\documentclass[twocolumn]{article}

\usepackage[accepted]{aistats2019}
\usepackage[top=1in, bottom=0.75in, left=0.9in, right=0.65in]{geometry}

\usepackage[round]{natbib}

\usepackage[utf8]{inputenc}
\usepackage{mathtools}
\usepackage{amsmath}
\usepackage{amsfonts}
\usepackage{amssymb}
\usepackage{amsthm}
\usepackage{booktabs}
\usepackage{color}
\usepackage[mathscr]{euscript}
\usepackage{dsfont}
\usepackage{comment}
\usepackage{hyperref}
\hypersetup{
    colorlinks,
    linkcolor={red!50!black},
    citecolor={blue!50!black},
    urlcolor={blue!80!black}
}

\usepackage{epsfig}
\usepackage{float}
\usepackage[justification=centering]{caption}
\usepackage{tikz, pgfplots}

\usepackage[Algorithm]{algorithm}
\usepackage{algorithmic}

\newtheorem{assumption}{Assumption}

\newtheorem{definition}{Definition}
\newtheorem{lemma}{Lemma}
\newtheorem{proposition}{Proposition}
\newtheorem{theorem}{Theorem}

\newcommand{\bE}{\mathbb{E}}
\newcommand{\bN}{\mathbb{N}}
\newcommand{\bP}{\mathbb{P}}
\newcommand{\bR}{\mathbb{R}}
\newcommand{\cB}{\mathcal{B}}
\newcommand{\cC}{\mathcal{C}}
\newcommand{\cS}{\mathcal{S}}
\newcommand{\cX}{\mathcal{X}}

\newcommand{\Y}[1]{Y^{(#1)}}

\newcommand{\Lb}{L_{\beta}}
\newcommand{\Cg}{C_{\gamma}}
\newcommand{\Ca}{C_{\alpha}}

\newcommand{\argmin}{\mathrm{argmin}}

\newcommand{\arginf}{\mathrm{arginf}}

\newcommand{\dist}{\mathrm{dist}}
\newcommand{\diam}{\mathrm{diam}}

\newcommand{\la}{\langle}
\newcommand{\ra}{\rangle}
\newcommand{\dd}{\, \mathrm{d}}
\newcommand{\norm}[1]{\left\Vert#1\right\Vert}
\newcommand{\normi}[1]{\norm{#1}_{\infty}}
\DeclarePairedDelimiter\abs{\lvert}{\rvert}
\newcommand{\st}{^{\star}}

\newcommand{\bigo}{\mathcal{O}}
\newcommand{\smallo}{{\scriptscriptstyle \mathcal{O}}}

\newcommand{\stcomp}[1]{{#1}^\complement}

\newcommand{\mubb}{\bar{\mu}(b)}
\newcommand{\lambb}{\bar{\lambda}(b)}
\newcommand{\Dk}{\Delta^K}
\newcommand{\hpb}{\hat{p}_b}

\newcommand{\pt}{\tilde{p}}

\DeclarePairedDelimiter\ceil{\lceil}{\rceil}
\DeclarePairedDelimiter\floor{\lfloor}{\rfloor}

\newcommand{\ie}{\textit{i.e.~}}

\newcommand\numberthis{\addtocounter{equation}{1}\tag{\theequation}}

\begin{document}

%

%

\twocolumn[

\aistatstitle{Regularized Contextual Bandits}

\aistatsauthor{ Xavier Fontaine \And Quentin Berthet \And  Vianney Perchet }

\aistatsaddress{
CMLA, ENS Cachan \\ CNRS, Université  Paris-Saclay
\And  Statistical Laboratory \\ DPMMS, University of Cambridge 
\And CMLA, ENS Cachan \\ CNRS, Université  Paris-Saclay \\
\& Criteo Research, Paris } ]

\begin{abstract}
We consider the stochastic contextual bandit problem with additional regularization. The motivation comes from problems where the policy of the agent must be close to some baseline policy known to perform well on the task. To tackle this problem we use a nonparametric model and propose an algorithm  splitting the context space into bins, solving simultaneously --- and independently --- regularized multi-armed bandit instances on each bin. We derive slow  and fast rates of convergence, depending on the unknown complexity of the problem. We also consider a new relevant margin condition to get problem-independent convergence rates, yielding intermediate rates interpolating between the aforementioned slow and fast rates.
\end{abstract}

\section{INTRODUCTION AND RELATED WORK}

In sequential optimization problems, an agent takes successive decisions in order to minimize an unknown loss function. An important class of such problems, nowadays known as bandit problems, has been mathematically formalized by~\citeauthor{robbins} in his seminal paper~\citep{robbins}. In the so-called stochastic multi-armed bandit problem, an agent chooses to sample (or ``pull'') among $K$ arms returning random rewards. Only the rewards of the selected arms are revealed to the agent who does not get any additional feedback. Bandits problems naturally model the exploration/exploitation trade-offs which arise in sequential decision making under uncertainty. Various general algorithms have been proposed to solve this problem, following the work of~\citet{lairobbins} who obtain a logarithmic regret for their sample-mean based policy. Further bounds have been obtained by~\citet{agrawal} and~\citet{ucb} who developed different versions of the well-known UCB algorithm.

The setting of classical stochastic multi-armed bandits is unfortunately too restrictive for real-world applications. The choice of the agent can and should be influenced by additional information (referred to as ``context'' or ``covariate") revealed by the environment. It encodes features having an impact on the arms' rewards. For instance, in  online advertising, the expected Click-Through-Rate depends on the identity,  the profile and the browsing history of the customer. These problems of bandits with covariates have been initially introduced by~\citet{woodroofe} and have attracted much attention since~\citet{side_observations,woodroofe_revisited}. This particular class of bandits problems is now known under the name of contextual bandits following~\citet{langford}.

Contextual bandits have been extensively studied in the last decades and several improvements upon  multi-armed bandits algorithms have been applied to contextual bandits~\citep{thompson_contextual,banditscovariates,policy_elimination}.
They are  quite intricate to study as they borrow aspects from both supervised learning and reinforcement learning. Indeed they use features to encode the context variables, as in supervised learning but also require an exploration phase to discover all the possible choices. Applications of contextual bandits are  numerous, ranging from online advertising~\citep{online_advertising}, to news articles recommendation~\citep{article_recommendation} or decision-making in the health and medicine sectors~\citep{mobile_health, online_medicine}.

Among the general class of stochastic multi-armed bandits, different settings can be studied. One natural hypothesis that can be made is to consider that the arms' rewards are regular functions of the context, \ie  two close context values have similar expected rewards. This setting  has been studied in~\citet{gp_ucb},~\citet{banditscovariates} and~\citet{slivkins}. A possible approach to this problem is to take inspiration from the regressograms used in nonparametric estimation~\citep{tsyb} and to divide the context space into several bins. This technique also used in online learning~\citep{online_bins} leads to the concept of UCBograms~\citep{zeevi} in bandits.

We introduce regularization to the problem of stochastic multi-armed bandits. It is a widely-used technique in  machine learning to avoid overfitting or to solve  ill-posed  problems. Here, the regularization forces the solution of the contextual bandits problem to be close to an existing known policy. As an example of motivation, an online-advertiser or any decision-maker may wish not to diverge too much from a handcrafted policy that is known to perform well. This has already motivated previous work such as Conservative Bandits~\citep{conservative}, where an additional arm corresponding to the handcrafted policy is added. By adding regularization, the agent can be sure to end up close to the chosen policy. Within this setting, the form of the objective function is not a classical bandit loss anymore, but contains a regularization term on the global policy. Regularized bandit problems, with no context, have been studied in~\citep{ucbfw}, with applications in online experiment design~\citep{BerCha16}, motivated by computational-statistical tradeoffs~\citep{BerRig13b,WanBerSam16,Ber14,WanBerPla16,BalBer18,BerEll19}. 

Our main contribution consists in an algorithm with proven slow or fast rates of convergence, depending on the unknown complexity of the problem at hand. These rates are better than the ones obtained for classical nonparametric contextual bandits. Based on nonparametric statistics we obtain parameter-independent intermediate convergence rates when the regularization function depends on the context value.

The remaining of this paper is organized as follows. We present the setting and problem in Section~\ref{sec:setting}. Our algorithm is described in Section~\ref{sec:algo}. Sections~\ref{sec:rates} and~\ref{sec:intermediate} are devoted to deriving the convergence rates. Lower bounds are detailed in Section~\ref{sec:lower} and experiments are presented in Section~\ref{sec:experiments}. Section~\ref{sec:ccl} concludes the paper.

\section{PROBLEM SETTING AND DEFINITIONS \label{sec:setting}}
\subsection{Problem Description \label{ssec:description}}

We consider a stochastic contextual bandits problem with $K \in \bN^*$ arms and time horizon $T$. It is defined as follows. At each time $t \in \{1, \dots, T\}$, Nature draws a context variable $X_t \in \cX = [0,1]^d$ uniformly at random. This context is revealed to an agent who chooses an arm $\pi_t$ amongst the $K$ arms. Only the loss $\Y{\pi_t}_t \in [0,1]$ is revealed to the agent.

For each arm $k \in \{1, \dots, K\}$ we note $\mu_k(X)\doteq\bE(\Y{k}|X)$ the conditional expectation of the arm's loss given the context. We impose classical regularity assumptions on the functions $\mu_k$ borrowed from nonparametric estimation. Namely we suppose that the functions $\mu_k$ are $(\beta, \Lb)$-Hölder, with $\beta \in (0,1]$. We note $\mathcal{H}_{\beta,\Lb}$ this class of functions.

\begin{assumption}[$\beta$-Hölder]
\label{ass:holder}
For all $k \in [K]$\footnote{$[K]=\{1,\cdots, K\}$},
\[ \forall x, y \in \cX, \ \abs{\mu_k(x)-\mu_k(y)}\leq \Lb \norm{x-y}_2^{\beta}. \]
\end{assumption}
We denote by $p: \cX \to \Delta^K$ the proportion function of each arm (also called occupation measure), where $\Delta^K$ is the unit simplex of $\bR^K$. In classical stochastic contextual bandits the goal of the agent is to minimize the following loss function
\[
L(p)=\int_{\cX} \la \mu(x), p(x)\ra \dd x.
\]

We add a regularization term representing the constraint on the optimal proportion function $p\st$. For example we may want to encourage $p\st$ to be close to a chosen proportion function $q$, or to be far from $\partial\Delta^K$. So we consider a convex regularization function $\rho : \Delta^K \times \cX \to \bR$, and a regularization parameter $\lambda : \cX \to \bR$. Both $\rho$ and $\lambda$ are known and given to the agent, while the $\mu_k$ functions are unknown and must be learned. We want to minimize the loss function
\[
L(p)=\int_{\cX} \la \mu(x), p(x) \ra + \lambda(x)\rho(p(x), x) \dd x.
\]

This is the most general form of the loss function. We study first the case where the regularization does not depend on the context (\ie when $\lambda$ is a constant and when $\rho$ is only a function of $p$).

The function $\lambda$ modulates the weight of the regularization and is chosen to be regular enough. More precisely we make the following assumption.
\begin{assumption}
\label{ass:reg}
$\lambda$ is a $\cC^{\infty}$ function and $\rho$ is a $\cC^1$ convex function. 
\end{assumption}
In order to prove some propositions, the convexity of $\rho$ will not be enough and we will need strong convexity.
We will also be led to consider $S$-smooth functions:

\begin{definition}
A continuously differentiable function $f$ defined on a set $\mathcal{D} \subset \bR^K$ is $S$-smooth (with $S>0$) if its gradient is $S$-Lipschitz continuous.
\end{definition}

The optimal proportion function is denoted by $p\st$ and verifies $p\st=\arginf_{p \in \{\cX \to\Delta^K\}}L(p)$. If an algorithm aiming at minimizing the loss $L$ returns a proportion function $p_T$ we define the regret as follows.

\begin{definition}
The regret of an algorithm outputting $p_T \in \{p:\cX \to \Delta^K\}$ is
\[
R(T) = \bE L(p_T) - L(p\st)
.
\]
\end{definition}
In the previous definition the expectation is taken on the choices of the algorithm. The goal is to find after $T$ samples a $p_T \in \{p:\cX \to \Delta^K\}$ the closest possible to $p\st$ in the sense of minimizing the regret. Note that $R(T)$ is actually a cumulative regret, since $p_T$ is the vector of the empirical frequency of each arm, \ie the normalized total number of pulls of each arm. Earlier choices affect this variable unalterably so that we face a trade-off between exploration and exploitation.

\subsection{Examples of Regularizations}

The most natural regularization function  considered throughout this paper is the (negative) entropy function defined as follows:
\[
\rho(p)=\sum_{i=1}^K p_i \log(p_i) \quad \textrm{for}\ p \in \Delta^K
.
\]
Since $\nabla^2_{ii} \rho(p)=1/p_i \geq 1$, $\rho$ is $1$-strongly convex. Using this function as a regularization forces $p$ to go to the center of the simplex, which means that each arm will be sampled a linear amount of time.

We can consider instead  the Kullback-Leibler divergence between $p$ and a known proportion function $q$:
\[
\rho(p)=D_{KL}(p||q)=\sum_{i=1}^K p_i \log\left(\dfrac{p_i}{q_i}\right) \quad \textrm{for}\ p \in \Delta^K
.
\]
Instead of pushing $p$ to the center of the simplex,  the KL divergence will push $p$ towards $q$. This is typically motivated by problems where the decision maker should not alter too much an existing policy $q$, known to perform well on the task. Another way to force $p$ to be close to a chosen policy $q$ is to use the $\ell^2$-regularization $\rho(p)=\norm{p-q}_2^2$. These two last examples have an explicit dependency on $x$ since $q$ depends on the context values, which was not the case of the entropy (which only depends on $x$ through $p$). Both the KL divergence and the $\ell^2$-regularization have a special form that allows us to remove this explicit dependency on $x$. They can indeed be written as
\[
\rho(p(x),x)=H(p(x))+\la p(x), k(x)\ra+c(x)
\]
with $H$ a $\zeta$-strongly convex function of $p$, $k$ a $\beta$-Hölder function of $x$ and $c$ any function of $x$.

Indeed,
\begin{align*}
D_{KL}(p||q)&=\sum_{i=1}^K p_i(x) \log\left(\dfrac{p_i(x)}{q_i(x)}\right) \\
&=\underbrace{\sum_{i=1}^K p_i(x)\log p_i(x)}_{H(p(x))}+ \la p(x), \underbrace{-\log q(x) \ra}_{k(x)}.
\end{align*}
And
\begin{align*}
\norm{p(x)-q(x)}_2^2=\underbrace{\norm{p(x)}^2}_{H(p(x))}+\la p(x), \underbrace{-2q(x)}_{k(x)} \ra + \underbrace{\norm{q(x)}^2}_{c(x)}.
\end{align*}
With this specific form the loss function writes as
\begin{align*}
L(p)&=\int_{\cX} \la \mu(x), p(x) \ra + \lambda(x)\rho(p(x),x) \dd x \\
&=\int_{\cX} \la \mu(x) + \lambda(x)k(x), p(x) \ra + \lambda(x)H(p(x)) \dd x \\&\phantom{=\int_{\cX} \la \mu(x) + \lambda(x)k(x), p(x) \ra i} +\int_{\cX} \lambda(x)c(x) \dd x.
\end{align*}
Since we aim at minimizing $L$ with respect to $p$, the last term $\int_{\cX} \lambda(x)c(x) \dd x$ is irrelevant for the minimization. Let us now note $\tilde{\mu} = \mu + \lambda k$. We are now minimizing
\[
\tilde{L}(p)=\int_{\cX} \la \tilde{\mu}(x), p(x) \ra + \lambda(x)H(p(x)) \dd x
.
\]
This is actually the standard setting of Subsection~\ref{ssec:description} with a regularization function $H$ independent of $x$. In order to preserve the regularity of $\tilde{\mu}$ we need $\lambda \rho$ to be $\beta$-Hölder which is the case if $q$ is sufficiently regular. Nonetheless, we remark that the relevant regularity is the one of $\mu$ since $\lambda$ and $\rho$ are known by the agent.

As a consequence, from now on we will only consider regularization functions $\rho$ that only depend on $p$.

\subsection{The Upper-Confidence Frank-Wolfe Algorithm \label{ssec:ucbfw}}

We now briefly present the Upper-Confidence Frank-Wolfe algorithm (UC-FW) from~\cite{ucbfw}, that will be an important tool of our own algorithm. This algorithm is designed to optimize an unknown convex function $L:\Dk \to \bR$. At each time step $t \geq 1$ the feedback available is a noisy estimate of $\nabla L(p_t)$, where $p_t$ is the vector of proportions of each action. The algorithm chooses the arm $k$ minimizing a lower confidence estimate of the gradient value (similarly as in the UCB algorithm~\citep{ucb}) and updates the proportions vector accordingly. Slow and fast rates for this algorithm are derived by the authors.

\section{ALGORITHM \label{sec:algo}}
\subsection{Idea of the Algorithm}

As the horizon is finite, even if we could use the  doubling-trick, and the  reward functions $\mu_k$ are smooth, we choose to split the context space $\cX$ into $B^d$ cubic bins of side size $1/B$. Inspired by UCBograms~\citep{zeevi} we are going to construct a (bin by bin)  piece-wise  constant  solution $\tilde{p}_T$.

We denote by $\cB$ the set of bins introduced. If $b \in \cB$ is a bin we note $|b|=B^{-d}$ its volume and $\diam(b)=\sqrt{d}/B$ its diameter. Since $\tilde{p}_T$ is piece-wise constant on each bin $b \in \cB$ (with value $\tilde{p}_T(b)$), we rewrite the loss function into
\begin{align*}
L(\pt_T)&=\int_{\cX} \la \mu(x), \pt_T(x) \ra + \lambda(x)\rho(\pt_T(x)) \dd x \\
&=\sum_{b \in \cB} \int_b \la \mu(x), \pt_T(b) \ra + \lambda(x)\rho(\pt_T(b)) \dd x \\
&=\dfrac{1}{B^d} \sum_{b \in \cB} \la \bar{\mu}(b), \pt_T(b) \ra + \bar{\lambda}(b)\rho(\pt_T(b)) \\
&=\dfrac{1}{B^d} \sum_{b \in \cB} L_b(\pt_T(b)) \numberthis \label{eq:sum_bin}
\end{align*}
where $L_b(p)=\la \bar{\mu}(b), p \ra + \bar{\lambda}(b)\rho(p)$ and $\mubb = \dfrac{1}{|b|}\int_b \mu(x)\dd x$ and $\lambb = \dfrac{1}{|b|}\int_b \lambda(x) \dd x$ are the mean values of $\mu$ and $\lambda$ on the bin $b$.

Consequently we just need to minimize the unknown convex loss functions $L_b$ for each bin $b \in \cB$. We fall precisely in the setting of Subsection~\ref{ssec:ucbfw} and we propose consequently the following algorithm: for each time step $t \geq 1$, given the context value $X_t$, we run one iteration of the UC-FW algorithm for the loss function $L_b$ corresponding to the bin $b \ni X_t$. We note $p_T(b)$ the results of the algorithm on each bin $b$.

\begin{algorithm}[h]
\caption{Regularized Contextual Bandits \label{algo}}
\algsetup{indent=2em}
\begin{algorithmic}[1]
\REQUIRE $K$ number of arms, $T$ time horizon
\REQUIRE $\mathcal{B}=\{1, \dots, B^d\}$ set of bins
\REQUIRE $\left(t\mapsto \alpha_k^{(b)}(t)\right)_{k\in[K]}^{b \in \cB}$ pre-sampling functions
\FOR{$b$ in $\cB$}
\STATE Sample $\alpha_k^{(b)}(T/B^d)$ times arm $k$ for all $k \in [K]$
\ENDFOR
\FOR{$t \geq 1$}
\STATE Receive context $X_t$ from the environment
\STATE $b_t \leftarrow$ bin of $X_t$
\STATE Perform one iteration of the UC-FW algorithm for the $L_{b_t}$ function on bin $b_t$
\ENDFOR
\RETURN the proportion vector $(p_T(1), \dots, p_T(B^d))$
\end{algorithmic}
\end{algorithm}
Line 2 of Algorithm~\ref{algo} consists in a pre-sampling stage where all arms are sampled a certain amount of time. 
It guarantees that $p_T(k)$ is bounded away from $0$ so that $p_T$ is bounded away from the boundary of $\Dk$, which will be required when $L_b$ is not smooth on $\partial\Dk$.


In the remaining of this paper, we derive slow and fast rates of convergence for this algorithm.

\subsection{Estimation and Approximation Errors}

In order to obtain a bound on the regret, we decompose it into an estimation error and an approximation error.

We note for all bins $b \in \cB$, $p\st_b=\arginf_{p \in \Delta^K} L_b(p)$ the minimum of $L_b$ on the bin $b$. We note $\tilde{p}\st$ the piece-wise constant function taking the values $p\st_b$ on the bin $b$.

The approximation error is the minimal achievable error within the class of piece-wise constant functions.

\begin{definition}
The approximation error $A(p)$ is the error between the best piece-wise constant function $\tilde{p}\st$ and the optimal solution $p\st$.
\[
A(p\st)=L(\tilde{p}\st) - L(p\st)
.
\]
\end{definition}
The estimation error is due to the errors made by the algorithm.

\begin{definition}
The estimation error $E(p_T)$ is the error between the result of the algorithm $p_T$ and the best piece-wise constant function $\tilde{p}\st$.
\[
E(p_T)=\bE L(p_T)-L(\tilde{p}\st) =\dfrac{1}{B^d}\sum_{b \in \cB} \bE L_b(p_T(b)) - L_b(p_b\st)
\]
\end{definition}
where the last equality comes from~\eqref{eq:sum_bin}.

We naturally have $R(T)=E(p_T)+A(p\st)$. In order to bound $R(T)$ we want to obtain bounds on both the estimation and the approximation error terms.

\section{CONVERGENCE RATES FOR CONSTANT $\lambda$ \label{sec:rates}}

In this section we consider the case where $\lambda$ is constant. We derive slow and fast rates of convergence. The proofs are deferred to Appendix~\ref{app:slow} and Appendix~\ref{app:fast}.

\subsection{Slow Rates \label{ssec:slow_rates}}

The analysis of the UC-FW algorithm gives the following bound.
\begin{proposition}
\label{prop:slow_estim}
Let $\rho$ be a $S$-smooth convex function on $\Dk$. If $p_T$ is the result of Algorithm~\ref{algo} and $\tilde{p}\st$ the best piece-wise constant function on the set of bins $\cB$, then the following bound on the estimation error holds\footnote{The Landau notation $\bigo(\cdot)$ has to be understood with respect to $T$. The precise bound is given in the proof.}
\[
\bE L(p_T)-L(\tilde{p}\st) =\bigo\left(\sqrt{K}B^{d/2}\sqrt{\dfrac{\log(T)}{T}}\right)
.
\]
\end{proposition}
Some regularization functions are not $S$-smooth on $\Dk$, for example the entropy whose Hessian is not bounded on $\Dk$. The following proposition shows that the previous result still holds, at least for the entropy.

\begin{proposition}
\label{prop:slow_entropy}
If $\rho$ is the entropy function the following bound on the estimation error holds
\[
\bE L(p_T(b))-L(\tilde{p}\st) \leq \bigo\left(B^{d/2}\dfrac{\log(T)}{\sqrt{T}}\right)
.
\]
\end{proposition}
The idea of the proof is to force the result of the algorithm to be ``inside" the simplex $\Dk$ (in the sense of the induced topology) by pre-sampling each arm.

In order to obtain a bound on the approximation error we notice that
\begin{align*}
L_b(p\st_b)&=\inf_{p \in \Delta^K} L_b(p)= \inf_{p \in \Delta^K} \lambda\rho(p)-\la-\bar{\mu}(b), p \ra \\
&= -(\lambda\rho)^*(-\bar{\mu}(b))=-\lambda\rho^*\left(-\dfrac{\mubb}{\lambda}\right)
\end{align*}
where $\rho^*$ is the Legendre-Fenchel transform of $\rho$.

Similarly,
\begin{align*}
&\int_b \la \mu(x), p\st(x) \ra + \lambda\rho(p\st(x)) \dd x \\
&\phantom{aaaaaa}= \int_b \inf_{p \in \Delta^K} -\la -\mu(x), p \ra + \lambda\rho(p) \dd x \\
&\phantom{aaaaaa}=\int_b-(\lambda\rho)^*(-\mu(x)) \dd x\\
&\phantom{aaaaaa}=\int_b-\lambda\rho^*\left(-\dfrac{\mu(x)}{\lambda}\right) \dd x.
\end{align*}

We want to bound
\begin{align*}
A(p\st)&=\sum_{b \in \cB} \int_{b} \la \mu(x), \pt\st(x) \ra + \lambda\rho(\pt\st(x)) \\
&\phantom{aaaaaaa}- \la \mu(x), p\st(x) \ra -\lambda\rho(p\st(x)) \dd x \\
&=\sum_{b\in\cB} \int_b \la \bar{\mu}(b), p\st_b \ra + \lambda\rho(p\st_b) \\
&\phantom{aaaaaaa}- \la \mu(x), p\st(x) \ra -\lambda\rho(p\st(x)) \dd x \\
&=\sum_{b\in\cB} \biggl( \int_b L_b(p\st_b) \dd x \\
&\phantom{aaaaaaa} - \int_b \la \mu(x), p\st(x) \ra + \lambda\rho(p\st(x)) \dd x \biggr) \\
&= \lambda\sum_{b\in\cB} \int_b \rho^*(-\mu(x)/\lambda)-\rho^*(-\bar{\mu}(b)/\lambda)\dd x \numberthis \label{eq:convconj}.
\end{align*}

With Equation~\eqref{eq:convconj} and convex analysis tools we prove the
\begin{proposition}
\label{prop:slow_approx}
If $\tilde{p}\st$ is the piece-wise constant function on the set of bins $\cB$ minimizing the loss function $L$, we have the following bound
\[
L(\tilde{p}\st) - L(p\st) \leq \sqrt{L_{\beta}Kd^{\beta}}B^{-\beta}
.
\]
\end{proposition}

Combining Propositions~\ref{prop:slow_estim} and~\ref{prop:slow_approx} we get the

\begin{theorem}[Slow rates]
\label{thm:slow}
If $\rho$ is a $S$-smooth convex function, applying Algorithm~\ref{algo} with choice $B=\Theta\left(\left(T/\log(T)\right)^{1/(2\beta+d)}\right)$ gives\footnote{The notation $\bigo_{\Lb,\beta,K,d}$ means that there is a hidden constant depending on $\Lb,\beta,K$ and $d$. The constant can be found in the proof in Appendix~\ref{app:slow}.}
\[
R(T) \leq \bigo_{\Lb,\beta,K,d}\left(\left(\dfrac{T}{\log(T)}\right)^{-\frac{\beta}{2\beta+d}}\right)
.
\]
\end{theorem}

Proposition~\ref{prop:slow_entropy} directly shows that the result of this theorem also holds when $\rho$ is the entropy function.

The detailed proof of the theorem (see Appendix~\ref{app:slow}) consists in choosing a value of $B$ balancing between the estimation and the approximation errors. Since $\beta \in (0,1]$, we see that the exponent of the convergence rate is below $1/2$ and that the proposed rate is slower than $T^{-1/2}$, hence the denomination of \textit{slow rate}.

When $\lambda=0$ we are in the usual contextual bandit setting. The propositions of this section hold and we recover the slow rates from~\cite{banditscovariates}.

\subsection{Fast Rates \label{ssec:fast_rates}}

We now consider possible fast rates, \ie convergence rates faster than $\bigo\left(T^{-1/2}\right)$. The price to pay to obtain these quicker rates compared to the ones from Subsection~\ref{ssec:slow_rates} is to have problem-dependent bounds, \ie convergence rates depending on the parameters of the problem, and especially on $\lambda$.

As in the previous section we can obtain a bound on the estimation error based on the convergence rates of the Upper-Confidence Frank-Wolfe algorithm.

\begin{proposition}
\label{prop:fast_estim}
If $\rho$ is $\zeta$-strongly convex and $S$-smooth, and if there exists $\eta>0$ such that for all $b\in \cB$, $\dist(p\st_b, \partial \Dk) \geq \eta$, then running Algorithm~\ref{algo} gives the estimation error
\[
\bE L(p_T)-L(\tilde{p}\st) =\bigo\left(B^d \left(S\lambda+\dfrac{K}{\lambda^2\zeta^2\eta^4}\right)\dfrac{\log^2(T)}{T}\right)
.
\]
\end{proposition}

This bound depends on several parameters of the problem: $\lambda$, distance $\eta$ of the optimum to the boundary of the simplex, strong convexity and smoothness constants. Since $\lambda$ can be arbitrarily small, $\eta$ can be small as well and $S$ large. Therefore the ``constant" factor can explode despite the convergence rate being ``fast": these terms describe only the dependency in $T$.

As in the previous section we want to consider regularization functions $\rho$ that are not smooth on $\partial\Dk$. To do so we force the vectors $p$ to be inside the simplex by pre-sampling all arms at the beginning of the algorithm. The following lemma shows that this is valid.

\begin{lemma}
\label{lemma:presampling}
On a bin $b$ if there exists $\alpha \in (0,1/2)$ and $p^o \in \Dk$ such that $p_b\st \succeq \alpha p^o$ (component-wise) then for all $i \in[K]$, the agent can safely sample arm $i$ $\alpha p^o_i T$ times at the beginning of the algorithm without changing the convergence results.
\end{lemma}

The intuition behind this lemma is that if all arms have to be sampled a linear amount of times to reach the optimum value, it is safe to pre-sample each of the arms linearly at the beginning of the algorithm. The goal is to ensure that the current proportion vector $p_t$ will always be far from the boundary in order to leverage the smoothness of $\rho$ in the interior of the simplex.

\begin{proposition}
\label{prop:fast_entropy}
If $\rho$ is the entropy function, sampling each arm $Te^{-1/\lambda}/K$ times during the presampling phase guarantees the same estimation error as in Proposition~\ref{prop:fast_estim} with constant $S=Ke^{1/\lambda}$.
\end{proposition}
In order to obtain faster rates for the approximation error we use Equation~\eqref{eq:convconj} and the fact that $\nabla \rho^*$ is $1/\zeta$-Lipschitz since $\rho$ is $\zeta$-strongly convex.

\begin{proposition}
\label{prop:fast_approx}
If $\rho$ is $\zeta$-strongly convex and if $\tilde{p}\st$ is the piece-wise constant function on the set of bins $\cB$ minimizing the loss function $L$, the following bound on the approximation error holds
\[
L(\tilde{p}\st) - L(p\st) \leq \dfrac{L_{\beta}Kd^{\beta}}{2\zeta\lambda} B^{-2\beta}
.
\]
\end{proposition}
Combining Propositions~\ref{prop:fast_estim} and~\ref{prop:fast_approx}, we obtain fast rates for our problem.
\begin{theorem}[Fast rates]
\label{thm:fast}
If $\rho$ is $\zeta$-strongly convex and if there exists $\eta>0$ such that for all $b\in \cB$, $\dist(p\st_b, \partial \Dk) \geq \eta$, applying Algorithm~\ref{algo} with the choice $B=\Theta\left(T/\log^2(T)\right)^{1/(2\beta+d)}$ gives the regret
\[
R(T) \leq \bigo_{\Lb,\beta,K,d,\lambda,\eta,\zeta, S}\left(\left(\dfrac{T}{\log^2(T)}\right)^{-\frac{2\beta}{2\beta+d}}\right)
.
\]
\end{theorem}

This rate matches the rates obtained in nonparametric estimation~\citep{tsyb}.
However, as shown in the proof presented in Appendix~\ref{app:fast}, this fast rate is obtained at the price of a factor involving $\lambda$, $\eta$ and $S$, which can be arbitrarily large. It is the goal of the next section to see how to remove this dependency in the parameters of the problem.

Proposition~\ref{prop:fast_entropy} shows that the previous theorem can also be applied to the entropy regularization.

\section{CONVERGENCE RATES FOR NON-CONSTANT $\lambda$}
\label{sec:intermediate}

In this section, we study the case where $\lambda$ is a function of the context value. This  is quite interesting as agents might want to modulate the weight of the regularization term depending on the context. All the proofs of this section can be found in Appendix~\ref{app:int}.

\subsection{Estimation and Approximation Errors}

 Equation~\eqref{eq:sum_bin} implies that the estimation errors obtained in Propositions~\ref{prop:slow_estim} and~\ref{prop:fast_estim} are still correct if $\lambda$ is replaced by $\lambb$. This is unfortunately not the case for the approximation error propositions because Equation~\eqref{eq:convconj} does not hold anymore. Indeed the approximation error becomes :
\begin{align*}
A(p\st)&=\sum_{b \in \cB} \int_{b} \la \mu(x), \pt\st(x) \ra + \lambda(x)\rho(\pt\st(x)) \\
&\phantom{aaaaaaaa}- \la \mu(x), p\st(x) \ra -\lambda(x)\rho(p\st(x)) \dd x \\
&=\sum_{b\in\cB} \int_b \la \bar{\mu}(b), p\st_b \ra + \lambda(x)\rho(p\st_b) \\
&\phantom{aaaaaaa}- \la \mu(x), p\st(x) \ra -\lambda(x)\rho(p\st(x)) \dd x \\
&=\sum_{b\in\cB} \biggl( \int_b L_b(p\st_b) \dd x \\
&\phantom{aaaaaa}- \int_b \la \mu(x), p\st(x) \ra + \lambda(x)\rho(p\st(x)) \dd x \biggr) \\
&= \sum_{b\in\cB} \int_b -(\bar{\lambda}(b)\rho)^*(-\bar{\mu}(b))+(\lambda(x)\rho)^*(-\mu(x)) \dd x \\
&=\sum_{b\in\cB} \int_b \lambda(x)\rho^*\left(-\dfrac{\mu(x)}{\lambda(x)}\right)-\bar{\lambda}(b)\rho^*\left(-\dfrac{\bar{\mu}(b)}{\bar{\lambda}(b)}\right) \dd x. \numberthis \label{eq:convconjvar}
\end{align*}
From this expression we obtain the following slow and fast rates of convergence. These rates are the same as in Section~\ref{sec:rates} in term of the powers of $B$ but have worse dependency in $\lambda$.

\begin{proposition}
\label{prop:slow_approx_var}
If $\rho$ is a strongly convex function and $\lambda$ a $\cC^{\infty}$ integrable non-negative function whose inverse is also integrable, we have on a bin $b$:
\begin{align*}
\int_b (\lambda(x)\rho)^*\left(-\mu(x)\right)-(\bar{\lambda}(b)\rho)^*\left(-\bar{\mu}(b)\right) \dd x \\ \leq\bigo(\Lb d^{\beta/2}B^{-\beta-d}).
\end{align*}
\end{proposition}

The important point is that the bound does not depend on $\lambda_{\min}$, which is not the case when we want to obtain fast rates for the approximation error:

\begin{proposition}
\label{prop:fast_approx_var}
If $\rho$ is a $\zeta$-strongly convex function and $\lambda$ a $\cC^{\infty}$ integrable non-negative function whose inverse is also integrable, we have on a bin $b$:
\begin{align*}
\int_b (\lambda(x)\rho)^*\left(-\mu(x)\right)-(\bar{\lambda}(b)\rho)^*\left(-\bar{\mu}(b)\right) \dd x \\
\leq \bigo\left( Kd\Lb^2\normi{\nabla \lambda}^2\dfrac{B^{-2\beta-d}}{\zeta\lambda_{\min}^3} \right).
\end{align*}
\end{proposition}

The rate in $B$ is improved compared to Proposition~\ref{prop:slow_approx_var} at the expense of the constant $1/\lambda_{\min}^3$ which can unfortunately be arbitrarily high.

\subsection{Margin Condition \label{ssec:margin_cond}}

We begin by giving a precise definition of the function $\eta$, the distance of the optimum to the boundary of $\Dk$.

\begin{definition}
\label{def:eta}
Let $x \in \cX$ a context value. We define by $p\st(x) \in \Dk$ the point where $p\mapsto \la \mu(x), p \ra + \lambda(x)\rho(p)$ attains its minimum, and
\[
\eta(x):=\dist(p\st(x),\partial \Dk).
\]
Similarly, if $p\st_b$ is the point where $L_b:p\mapsto \la \mubb, p\ra+\lambb \rho(p)$ attains its minimum, we define
\[
\eta(b):=\dist(p\st_b,\partial \Dk)
.
\]
\end{definition}
The fast rates obtained in Subsection~\ref{ssec:fast_rates} provide good theoretical guarantees but may be useless in practice since they depend on a constant that can be arbitrarily large. We would like to  discard the dependency on the parameters, and especially $\lambda$ (that controls $\eta$ and $S$).

Difficulties arise when $\lambda$ and $\eta$ take values that are very small, meaning for instance that we consider nearly no regularization. This is not likely to happen since we do want to study contextual bandits with regularization. To formalize that we make an additional assumption, which is common in nonparametric regression~\citep{tsyb} and is known as a \textit{margin condition}:

\begin{assumption}[Margin Condition]
\label{ass:margin}
We assume that there exist $\delta_1>0$ and $\delta_2>0$ as well as $\alpha>0$ and $C_m>0$ such that
\begin{align*}
&\forall \delta \in (0, \delta_1], \ \bP_X(\lambda(x) < \delta) \leq C_m \delta^{6\alpha} \\
\quad  \textrm{and} \quad
&\forall \delta \in (0, \delta_2], \ \bP_X(\eta(x) < \delta) \leq C_m \delta^{6\alpha}
.
\end{align*}
\end{assumption}
The non-negative parameter $\alpha$  controls the importance of the margin condition.

The margin condition  limits the number of bins on which $\lambda$ or $\eta$ can be small. Therefore we split the bins of $\cB$ into two categories, the ``well-behaved bins" on which $\lambda$ and $\eta$ are not too small, and the ``ill-behaved bins" where $\lambda$ and $\eta$ can be arbitrarily small. The idea is to use the fast rates on the ``well-behaved bins" and the slow rates (independent of $\lambda$ and $\eta$) on the ``ill-behaved bins". This is the point of Subsection~\ref{ssec:int}.

Let $C_L= \sqrt{\frac{K}{K-1}}\frac{\normi{\lambda}+\normi{\nabla\lambda}}{\zeta}$, $c_1=1+\norm{\nabla\lambda}_{\infty}d^{\beta/6}$ and $c_2=1+C_L d^{\beta/2}$.

We define the set of ``well-behaved bins" $\mathcal{WB}$ as 
\begin{align*}
\mathcal{WB}=\{ b \in \cB, \ \exists \ x_1 \in b, \ \lambda(x_1)\geq c_1B^{-\beta/3} \\
\textrm{ and } \exists \ x_2 \in b, \ \eta(x_2)\geq c_2B^{-\beta/3}\},
\end{align*}
and the set of ``ill-behaved bins" as its complementary set in $\cB$.

With the smoothness and regularity Assumptions~\ref{ass:holder} and~\ref{ass:reg}, we derive lower bounds for $\lambda$ and $\eta$ on the ``well-behaved bins".
\begin{lemma}
\label{lemma:wb}
If $b$ is a well-behaved bin then
\[
\forall x \in b, \ \lambda(x) \geq B^{-\beta/3}
\quad \textrm{and} \quad
\forall x \in b, \ \eta(x) \geq B^{-\beta/3}
.
\]
\end{lemma}

\subsection{Intermediate Rates \label{ssec:int}}

We summarize the different error rates obtained in the previous sections.

\begin{table}[H]
\caption{Slow and Fast Rates for Estimation and Approximation Errors on a Bin} \label{table:summary}
\begin{center}
\begin{tabular}{c|c|c}
\textbf{Error}  & \textbf{Slow} & \textbf{Fast} \\
\midrule
\textbf{Estim.}  & $B^{-d/2}\sqrt{\dfrac{\log(T)}{T}}$ & $\dfrac{\log^2(T)}{T}\left(S\lambda+\dfrac{1}{\eta^4 \lambda^2}\right)$ \\
\textbf{Approx.} & $B^{-d}B^{-\beta}$ & $\dfrac{B^{-2\beta-d}}{\lambda^3}$ \\
\midrule
\textbf{$B$} & $\left(\dfrac{T}{\log(T)}\right)^{\frac{1}{2\beta+d}}$ & $\left(\dfrac{T}{\log^2(T)}\right)^{\frac{1}{2\beta+d}}$\\
\midrule
\textbf{$R(T)$} & $\left(\dfrac{T}{\log(T)}\right)^{\frac{-\beta}{2\beta+d}}$ & $\left(\dfrac{T}{\log^2(T)}\right)^{\frac{-2\beta}{2\beta+d}}$
\end{tabular}
\end{center}
\end{table}
For the sake of clarity we removed the dependency on the bin, writing $\lambda$ instead of $\lambb$, and we only kept the relevant constants, that can be very small ($\lambda$ and $\eta$), or very large ($S$). 

Table~\ref{table:summary} shows that the slow rates do not depend on the constants, so that we can use them on the ``ill-behaved bins".

\begin{theorem}[Intermediate rates]
\label{thm:int}
Applying Algorithm~\ref{algo} with an entropy regularization and margin condition with parameter $\alpha \in(0,1)$, the choice $B=\Theta\left(T/\log^2(T)\right)^{\frac{1}{2\beta+d}}$ leads to the regret
\[R(T) = \bigo_{K,d,\alpha,\beta,\Lb} \left(\dfrac{T}{\log^2(T)}\right)^{-\frac{\beta}{2\beta+d}(1+\alpha)}.\]
\end{theorem}
As explained in the proof (Appendix~\ref{app:int}), we  use a pre-sampling stage on each bin to force the entropy to be smooth, as in the proofs of Propositions~\ref{prop:slow_entropy} and~\ref{prop:fast_entropy}.

We consider now the extreme values of $\alpha$.
If $\alpha\to 0$, there is no margin condition and the speed obtained is $T^{-\frac{\beta}{2\beta+d}}$ which is exactly the slow rate from Theorem~\ref{thm:slow}. If $\alpha \to 1$, there is a strong margin condition and the rate of Theorem~\ref{thm:int} tends to $T^{-\frac{2\beta}{2\beta+d}}$ which is the fast rate from Theorem~\ref{thm:fast}. Consequently we get that the intermediate rates from Theorem~\ref{thm:int} do interpolate between the slow and fast rates obtained previously.

\section{LOWER BOUNDS}
\label{sec:lower}

The results in Theorems~\ref{thm:slow} and~\ref{thm:fast} have optimal exponents in the dependency in $T$. For the slow rate, since the regularization can be equal to 0, or a linear form, the lower bounds on contextual bandits in this setting apply~\citep{audibert_tsyb,zeevi}, matching this upper bound. For the fast rates, the following lower bound holds, based on a reduction to nonparametric regression~\citep{tsyb,gyorfi}.

\begin{theorem}
\label{thm:lowfast}
For any algorithm with bandit input and output $\hat p_T$, for $\rho$ that is 1-strongly convex, we have
\[
\inf_{\hat p} \sup_{\substack{\mu \in \mathcal{H}_\beta\\ \rho \in \text{1-str. conv.}}}\Big\{ \bE[L(\hat p_T)] - L(p\st) \Big\} \ge C\,  T^{-\frac{2\beta}{2\beta+d}} \, \, ,\]
for a universal constant $C$.
\end{theorem}

The proof is in Appendix~\ref{app:lb}. The upper and lower bound match up to logarithmic terms. This bound is obtained for $K=2$, and the dependency of the rate in $K$ is not analyzed here.

\section{EMPIRICAL RESULTS}
\label{sec:experiments}

We present in this section experiments and simulations for the regularized contextual bandits problem. The setting we consider uses $K=3$ arms, with an entropy regularization and a fixed parameter $\lambda=0.1$. We run successive experiments for values of $T$ ranging from $1\,000$ to $100\, 000$, and for different values of the smoothness parameter $\beta$. The arms' rewards follow $3$ different probability distributions (Poisson, exponential and Bernoulli), with $\beta$-Hölder mean functions.

The results presented in Figure~\ref{fig:regret} shows that $T\mapsto T\cdot R(T)$ growths as expected, and the lower $\beta$, the slower the convergence rate,  as shown on the graph.
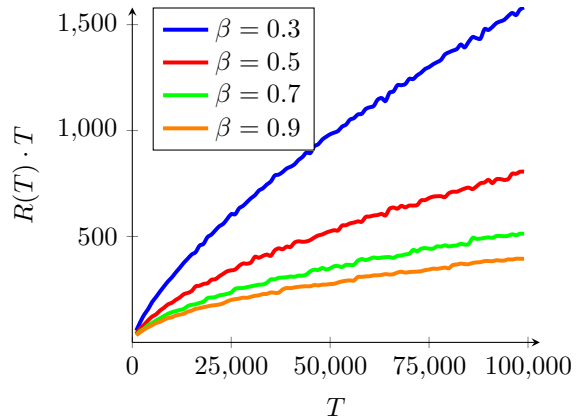
\begin{figure}[H]
\begin{center}
\begin{tikzpicture}[domain=0:5]
\begin{axis}
        [
        ,axis x line=bottom
  		,axis y line=center
		,ylabel near ticks,
		,xlabel near ticks,
        ,width=7cm
        ,xlabel=$T$
        ,ylabel=$R(T)\cdot T$
        ,xmin=0
        ,xmax=103000
        ,ymin=0
        ,scaled ticks=false
        ,xtick={0,25000,50000,75000,100000}
        ,xticklabels={0,$\textrm{25,000}$,$\textrm{50,000}$,$\textrm{75,000}$,$\textrm{100,000}$}
        ,legend style={at={(0.05,1)},anchor=north west}
        ]
        \addplot[smooth,blue,line width=1.5pt] table [x index={1}, y index={2}, col sep=comma] {file_0.30.csv};
        \addlegendentry{$\beta=0.3$};
        \addplot[smooth,red,line width=1.5pt] table [x index={1}, y index={2}, col sep=comma] {file_0.50.csv};
        \addlegendentry{$\beta=0.5$};
        \addplot[smooth,green,line width=1.5pt] table [x index={1}, y index={2}, col sep=comma] {file_0.70.csv};
        \addlegendentry{$\beta=0.7$};
        \addplot[smooth,orange,line width=1.5pt] table [x index={1}, y index={2}, col sep=comma] {file_0.90.csv};
        \addlegendentry{$\beta=0.9$};
    \end{axis}
\end{tikzpicture}
\caption{Regret as a Function of $T$\label{fig:regret}}
\end{center}
\end{figure}
In order to verify that the fast rates proven in Subsection~\ref{ssec:fast_rates} are indeed reached, we plot on Figure~\ref{fig:norm_regret} the ratio between the regret and the theoretical bound on the regret $\left(T/\log^2(T)\right)^{-\frac{2\beta}{2\beta+d}}$. We observe that this ratio is approximately constant as a function of $T$, which validates empirically the theoretical convergence rates.
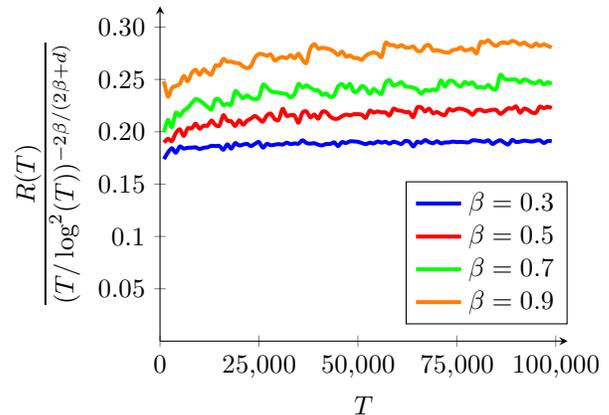
\begin{figure}[H]
\begin{center}
\begin{tikzpicture}[domain=0:5]
\begin{axis}
        [
        ,axis x line=bottom
  		,axis y line=center
		,ylabel near ticks,
		,xlabel near ticks,
        ,width=7cm
        ,xlabel=$T$
        ,ylabel=$\dfrac{R(T)}{(T/\log^2(T))^{-2\beta/(2\beta+d)}}$
        ,xmin=0
        ,xmax=103000
        ,ymin=0
        ,ymax=0.32
        ,scaled ticks=false
        ,xtick={0,25000,50000,75000,100000}
        ,xticklabels={0,$\textrm{25,000}$,$\textrm{50,000}$,$\textrm{75,000}$,$\textrm{100,000}$}
        ,ytick={0,0.05,...,0.35}
        ,yticklabels={0,0.05,0.1,0.15,0.20,0.25,0.30}
        ,legend style={at={(1,0.05)},anchor=south east}
        ]
        \addplot[smooth,blue,line width=1.5pt] table [x index={1}, y index={3}, col sep=comma] {file_0.30.csv};
        \addlegendentry{$\beta=0.3$};
        \addplot[smooth,red,line width=1.5pt] table [x index={1}, y index={3}, col sep=comma] {file_0.50.csv};
        \addlegendentry{$\beta=0.5$};
        \addplot[smooth,green,line width=1.5pt] table [x index={1}, y index={3}, col sep=comma] {file_0.70.csv};
        \addlegendentry{$\beta=0.7$};
        \addplot[smooth,orange,line width=1.5pt] table [x index={1}, y index={3}, col sep=comma] {file_0.90.csv};
        \addlegendentry{$\beta=0.9$};
    \end{axis}
\end{tikzpicture}
\caption{Normalized Regret as a Function of $T$\label{fig:norm_regret}}
\end{center}
\end{figure}

\section{CONCLUSION}
\label{sec:ccl}
We proposed an algorithm for the problem of contextual bandits with regularization reaching fast rates similar to the ones obtained in nonparametric estimation, and validated by our experiments. We can discard the parameters of the problem in the convergence rates by applying a margin condition that allows us to derive intermediate convergence rates interpolating perfectly between the slow and fast rates.

\subsubsection*{Acknowledgments}

Xavier Fontaine was supported by grants from Région Ile-de-France. Quentin Berthet was supported by The Alan Turing Institute under the EPSRC grant EP/N510129/1. Vianney Perchet has benefited from the support of the FMJH Program Gaspard Monge in optimization and operations research (supported in part by EDF), from the Labex LMH and from the CNRS through the PEPS program.

\bibliography{contextual_bandits}

\onecolumn


\appendix
\section{PROOFS OF SLOW RATES}
\label{app:slow}

We prove in this section the propositions and theorem of Subsection~\ref{ssec:slow_rates}.

We begin by a lemma on the concentration of $T_b$, the number of context samples falling in a bin $b$.

\begin{lemma}
\label{lemma:chernoff}
For all $b \in \cB$, let $T_b$ the number of context samples falling in the bin $b$. We have
\[
\bP\left(\exists b \in \cB, \ \abs*{T_b-\dfrac{T}{B^d}} \geq \dfrac{1}{2}\dfrac{T}{B^d}\right) \leq 2B^d \exp\left(-\dfrac{T}{12B^d}\right).
\]
\end{lemma}

\begin{proof}
For a bin $b\in\cB$ and $t \in \{1,\dots,T\}$, let $Z^{(b)}_t=\mathds{1}_{\{X_t \in \cB\}}$ which is a random Bernoulli variable of parameter $1/B^d$.

We have $T_b=\sum_{t=1}^T Z_t$ and $\bE[T_b]=T/B^d$.

Using a multiplicative Chernoff's bound~\citep{hdp} we obtain:
\[
\bP\left( \abs*{T_b-\bE[T_b]} \geq \dfrac{1}{2}\bE[T_b] \right) \leq 2\exp\left(-\dfrac{1}{3}\left(\dfrac{1}{2}\right)^2\dfrac{T}{B^d}\right)=2\exp\left(-\dfrac{T}{12B^d}\right)
.
\]
We conclude with an union bound on all the bins.
\end{proof}

\begin{proof}[Proof of Proposition~\ref{prop:slow_estim}]
We have
\[
E(p_T)=\bE L(p_T)-L(\tilde{p}\st) =\dfrac{1}{B^d}\sum_{b \in \cB} \bE L_b(p_T(b)) - L_b(p_b\st)
\]
Let us now consider a single bin $b \in \cB$. We have run the UCB Frank-Wolfe~\citep{ucbfw} algorithm for the function $L_b$ on the bin $b$ with $T_b$ iterations.

For all $p \in \Dk$, $L_b(p)=\la \mubb, p\ra+\lambda \rho(p)$, then for all $p \in \Delta^K$, $\nabla L_b(p)=\mubb + \lambda \nabla \rho(p)$ and $\nabla^2 L_b(p)=\lambda \nabla^2 \rho(p)$. Since $\rho$ is a $S$-smooth convex function, $L_b$ is a $\lambda S$-smooth convex function.

We consider the event $A$:
\[
A\doteq\left\lbrace \forall b \in \cB, \ T_b \in \left[\dfrac{T}{2B^d},\dfrac{3T}{2B^d}\right]  \right\rbrace.
\]

Lemma~\ref{lemma:chernoff} shows that $\bP(\stcomp{A}) \leq 2B^d\exp\left(-\dfrac{T}{12B^d}\right)$.

Theorem 3 of~\citet{ucbfw} shows that, on event $A$:
\begin{align*}
\bE L_b(p_T(b))-L_b(p\st_b) &\leq 4\sqrt{\dfrac{3K\log(T_b)}{T_b}}+\dfrac{S\log(eT_b)}{T_b}+\left(\dfrac{\pi^2}{6}+K\right)\dfrac{2\normi{\nabla L_b}+\normi{L_b}}{T_b} \\
&\leq 4\sqrt{\dfrac{6K\log(T)}{T/B^d}}+\dfrac{2S\log(eT)}{T/B^d}+2\left(\dfrac{\pi^2}{6}+K\right)\dfrac{2\normi{\nabla L_b}+\normi{L_b}}{T/B^d}
.
\end{align*}

Since $\rho$ is of class $\cC^1$, $\rho$ and $\nabla \rho$ are bounded on the compact set $\Dk$. It is also the case for $L_b$ and consequently $\normi{L_b}$ and $\normi{\nabla L_b}$ exist and are finite and can be expressed in function of $\normi{\rho}$, $\normi{\nabla \rho}$ and $\normi{\lambda}$.
On event $\stcomp{A}$, $\bE L_b(p_T(b))-L_b(p\st_b) \leq 2\normi{L_b} \leq 2+2\normi{\lambda\rho}$.

Summing over all the bins in $\cB$ we obtain:
\[
\bE L(p_T)-L(p\st) \leq 4B^{d/2}\sqrt{\dfrac{6K\log(T)}{T}}+B^d\dfrac{2S\log(eT)}{T}+4KB^d\dfrac{4+2\normi{\lambda \nabla \rho}+\normi{\lambda \rho}}{T} + 4B^d(1+\normi{\lambda\rho})e^{-\frac{T}{12B^d}}  \numberthis \label{eq:slow_ucbfw}
.\]

The first term of Equation~\eqref{eq:slow_ucbfw} dominates the others and we can therefore write that
\[
\bE L(p_T)-L(p\st) = \bigo \left( \sqrt{K} B^{d/2}\sqrt{\dfrac{\log(T)}{T}} \right)
\]
where the $\bigo$ is valid for $T \to \infty$.

\end{proof}

\begin{proof}[Proof of Proposition~\ref{prop:slow_entropy}]
We consider a bin $b \in \cB$ containing $t$ samples.

Let $\cS \doteq \left\lbrace p \in \Delta^K \ | \ \forall i \in [K],\ p_i \geq \dfrac{\lambda}{\sqrt{t}} \right\rbrace$. In order to force all the successive estimations of $p\st_b$ to be in $\cS$ we sample each arm $\lambda \sqrt{t}$ times. Thus we have $\forall i \in [K], \ p_i \geq \lambda/\sqrt{t}$. Then we apply the UCB-Frank Wolfe algorithm on the bin $b$. Let
\[
\hpb \doteq \min_{p \in \cS} L_b(p) \quad \textrm{and}\quad p\st_b\doteq \min_{p \in \Dk} L_b(p).
\]

\begin{itemize}
\item \textbf{Case 1:} $\hpb=p\st_b$,  \ie the minimum of $L_b$ is in $\cS$.

For all $p \in \Dk$, $L_b(p)=\la \mubb, p\ra+\lambda \rho(p)$, then for all $p \in \Delta^K$, $\nabla L_b(p)=\mubb + \lambda(1+\log(p))$ and $\nabla^2_{ii}L_b(p)=\lambda/p_i$. Therefore on $\cS$ we have
\[
\nabla^2_{ii}L_b(p)\leq \sqrt{t}.
\]
And consequently $L_b$ is $\sqrt{t}$-smooth. And since $\nabla_i L_b(p)=1+\lambda \log(p_i)$, $\normi{\nabla L_b(p)} \lesssim \log(t)$. We can apply the same steps as in the proof of Proposition~\ref{prop:slow_estim} to find that
\[
\bE L_b(p_t(b))-L_b(p\st_b) \leq 4\sqrt{\dfrac{3K\log(t)}{t}}+\dfrac{\sqrt{t}\log(et)}{t}+\left(\dfrac{\pi^2}{6}+K\right)\dfrac{2\log(t)+\log(K)}{t}=\bigo\left(\dfrac{\log(t)}{\sqrt{t}}\right)
.
\]

\item \textbf{Case 2:} $\hpb\neq p\st_b$. By strong convexity of $L_b$, $\hpb$ cannot be a local minimum of $L_b$ and therefore $\hpb \in \partial \Dk$.

The Case 1 shows that
\[
\bE L_b(p_t(b))-L_b(\hpb) \leq \bigo\left(\dfrac{\log(t)}{\sqrt{t}}\right)
.
\]

Let $\pi = (\pi_1, \dots, \pi_K)$ with $\pi_i \doteq \max(\lambda/\sqrt{t},\hat{p}_{b,i})$. We have $\norm{\pi-\hpb}_2 \leq \sqrt{K}\lambda/\sqrt{t}$.

Let us derive an explicit formula for $p_b\st$ knowing the explicit expression of $\rho$. In order to find the optimal $\rho\st$ value let us minimize $(p\mapsto L_b(p))$ under the constraint that $p$ lies in the simplex $\Delta^K$. The KKT equations give the existence of $\xi \in \bR$ such that for each $i \in [K]$, $\bar{\mu}_i(b)+\lambda\log(p_i)+\lambda+\xi=0$ which leads to $p\st_{b,i}=e^{-\bar{\mu}_i(b)/\lambda}/Z$ where $Z$ is a normalization factor. Since $Z=\sum_{i=1}^K e^{-\bar{\mu}_i(b)/\lambda}$ we have $Z \leq K$ and $p\st_{b,i} \geq e^{-1/\lambda}/K$. Consequently for all $p$ on the segment between $\pi$ and $p_b\st$ we have $p_i\geq e^{-1/\lambda}/K$ and therefore $\lambda(1+\log(p_i)) \geq \lambda(1-\log K) -1$ and finally $\abs{\nabla_i L_b(p)} \leq 4 \normi{\lambda}\log(K)$.

Therefore $L_b$ is $4\sqrt{K}\log(K)$-Lipschitz and
\[
\norm{L_b(p\st_b)-L_b(\pi)}_2\leq 4\normi{\lambda}\sqrt{K}\log(K)\norm{\pi-\hpb}_2 \leq 4K\log(K)\normi{\lambda}^2/\sqrt{t}=\bigo(1/\sqrt{t}).
\]

Finally, since $L_b(\pi) \geq L_b(\hpb)$ (because $\pi \in \cS$), we have
\[
\bE L_b(p_t(b))-L_b(p\st_b) \leq \bE L_b(p_t(b))-L_b(\hpb)+L_b(\hpb)-L_b(p\st_b) \leq \bigo\left(\dfrac{\log(t)}{\sqrt{t}}\right) + L(\pi)-L(p_b\st) = \bigo\left(\dfrac{\log(t)}{\sqrt{t}}\right).
\]

We conclude by summing on the bins and using that $t\in[T/2B^d,3T/2B^d]$ with high probability, as in the proof of Proposition~\ref{prop:slow_estim}.

\end{itemize}

\end{proof}

\begin{proof}[Proof of Proposition~\ref{prop:slow_approx}]
We have to bound the quantity
\[
L(\tilde{p}\st) - L(p\st)= \lambda\sum_{b\in\cB} \int_b \rho^*(-\mu(x)/\lambda)-\rho^*(-\bar{\mu}(b)/\lambda)\dd x
.
\]

Classical results on convex conjugates~\citep{hiriart1} give that $\nabla \rho^*(y)=\argmin_{x \in \Delta^K} \rho(x)-\la x,y\ra$ for all $y \in \bR^K$. Consequently, $\nabla \rho^*(y) \in \Delta^K$ and for all $y \in \bR^K$, $\norm{\nabla \rho^*(y)} \leq 1$ showing that $\rho^*$ is $1$-Lipschitz continuous.
This leads to
\begin{align*}
L(\pt\st)-L(p\st)&\leq \lambda\sum_{b\in\cB} \int_b \norm{\dfrac{\mu(x)-\bar{\mu}(b)}{\lambda}} \dd x \\
&\leq \sum_{b\in\cB} \int_b \sqrt{L_{\beta}K}\left(\dfrac{\sqrt{d}}{B} \right)^{\beta} \dd x \\
&\leq \sqrt{L_{\beta}Kd^{\beta}}B^{-\beta}
\end{align*}
because all the $\mu_k$ are $(\Lb, \beta)$-Hölder.
\end{proof}

\begin{proof}[Proof of Theorem~\ref{thm:slow}]
We will denote by $C_k$ with increasing values of $k$ the constants.
Since the regret is the sum of the approximation error and the estimation error we obtain
\[
R(T) \leq \sqrt{\Lb d^{\beta} K} B^{-\beta} + C_1\sqrt{K}B^{d/2}\sqrt{\dfrac{\log(T)}{T}} + B^d\dfrac{2S\log(eT)}{T}+C_2 K\dfrac{B^d}{T} + 4B^d(1+\normi{\lambda\rho})\exp\left(-\dfrac{T}{12B^d}\right).
\]
With the choice of
\[
B=\left(C_2\beta\sqrt{\Lb}d^{\beta/2-1}\right)^{1/(\beta+d/2)}\left(\dfrac{T}{\log(T)}\right)^{1/(2\beta+d)},
\]
we find that the three last terms of the regret are negligible with respect to the first two.
This gives
\[
R(T) \leq \bigo\left(\left(3\sqrt{K}\Lb^{d/(4\beta+2d)}d^{\beta(4+d)/(4\beta+2d)}(C_2\beta)^{-\beta/(2\beta+d)} \right) \left(\dfrac{T}{\log(T)}\right)^{-\beta/(2\beta+d)}\right).
\]
\end{proof}

\section{PROOFS OF FAST RATES}
\label{app:fast}

We prove now the propositions and theorem of Subsection~\ref{ssec:fast_rates}.

\begin{proof}[Proof of Proposition~\ref{prop:fast_estim}]
The proof is very similar to the one of Proposition~\ref{prop:slow_estim}. We decompose the estimation error on the bins:
\[
\bE L(p_T)-L(\tilde{p}\st) =\dfrac{1}{B^d}\sum_{b \in \cB} \bE L_b(p_T(b)) - L_b(p_b\st)
.
\]
Let us now consider a single bin $b \in \cB$. We have run the UCB Frank-Wolfe algorithm for the function $L_b$ on the bin $b$ with $T_b$ samples.

As in the proof of Proposition~\ref{prop:slow_estim} we consider the event $A$.

Theorem 7 of~\citet{ucbfw}, applied to $L_b$ which is a $\lambda S$-smooth $\lambda \zeta$-strongly convex function, shows that on event $A$:
\[
\bE L(p_T) - L(p\st) \leq 2\tilde{c}_1 \dfrac{\log^2(T)}{T/B^d}+2\tilde{c}_2\dfrac{\log(T)}{T/B^d}+\tilde{c}_3\dfrac{2}{T/B^d}
\]
with $\tilde{c}_1=\dfrac{96K}{\zeta \lambda \eta^2}$, $\tilde{c}_2=\dfrac{24}{\zeta \lambda \eta^3}+\lambda S$ and $\tilde{c}_3=24\left(\dfrac{20}{\zeta\lambda\eta^2}\right)^2K+\dfrac{\lambda\zeta\eta^2}{2}+\lambda S$.
Consequently

\[
\bE L(p_T) - L(p\st) \leq 2\tilde{c}_1 \dfrac{\log^2(T)}{T/B^d}+2\tilde{c}_2\dfrac{\log(T)}{T/B^d}+\tilde{c}_3\dfrac{2}{T/B^d} + 4B^d(1+\normi{\lambda\rho})\exp\left(-\dfrac{T}{12B^d}\right)
.
\]

In order to have a simpler expression we can use the fact that $\lambda$ and $\eta$ are constants that can be small while $S$ can be large. Consequently $\tilde{c}_3$ is the largest constant among $\tilde{c}_1$, $\tilde{c}_2$ and $\tilde{c}_3$ and we obtain

\[
\bE L(p_T) - L(p\st) \leq \bigo\left(\left(\dfrac{K}{\lambda^2 \zeta^2 \eta^4}+S\lambda\right) B^d\dfrac{\log^2(T)}{T}\right),
\]
because the other terms are negligible.
\end{proof}

\begin{proof}[Proof of Lemma~\ref{lemma:presampling}]
We consider a single bin $b \in \cB$.
Let us consider the function
\[
\hat{L}_b:p\mapsto L_b(\alpha p^o+(1-\alpha)p)
.
\]
Since for all $i$, $p\st_{b,i} \geq \alpha p_i^o$ and since $\Dk$ is convex we know that $\min_{p \in \Dk}\hat{L}_b(p)=L_b(p_b\st)$.

If $p$ is the frequency vector obtained by running the UCB-Frank Wolfe algorithm for function $\hat{L}_b$ with $(1-\alpha)T$ samples then minimizing $\hat{L}_b$ is equivalent to minimizing $L$ with a presampling stage.

Consequently the whole analysis on the regret still holds with $T$ replaced by $(1-\alpha)T$. Thus fast rates are kept with a constant factor $1/(1-\alpha) \leq 2$.
\end{proof}

\begin{proof}[Proof of Proposition~\ref{prop:fast_entropy}]
For the entropy regularization, we have
\[
p\st_{b,i}=\dfrac{\exp(-\mubb_i/\lambda)}{\sum_{j=1}^K \exp(-\mubb_j/\lambda)}\leq \dfrac{\exp(-1/\lambda)}{K}.
\]

We apply Lemma~\ref{lemma:presampling} with $p^o=\left(\dfrac{1}{K},\dots,\dfrac{1}{K}\right)$ and $\alpha=\exp(-1/\lambda)$. Consequently each arm is presampled $T\exp(-1/\lambda)/K$ times and finally we have
\[
\forall i \in [K], p_i \geq \dfrac{\exp(-1/\lambda)}{K}.
\]
Therefore we have
\[
\forall i \in[K], \ \nabla_{ii}\rho(p)=\dfrac{1}{p_i}\leq K\exp(1/\lambda),
\]
showing that $\rho$ is $K\exp(1/\lambda)$-smooth.
\end{proof}

In order to prove the Proposition~\ref{prop:fast_approx} we will need the following lemma which is a direct consequence of a result on smooth convex functions.

\begin{lemma}
\label{lemma:conv}
Let $f:\bR^d \to \bR$ be a convex function of class $\cC^1$ and $L>0$. Let $g:\bR^d \ni x \mapsto \dfrac{L}{2}\norm{x}^2-f(x)$. Then $g$ is convex if and only if $\nabla f$ is $L$-Lipschitz continuous.
\end{lemma}

\begin{proof}
Since $g$ is continuously differentiable we can write
\begin{align*}
g\textrm{ convex } &\Leftrightarrow \forall x,y \in \bR^d, \ g(y) \geq g(x) + \la \nabla g(x), y-x \ra \\
&\Leftrightarrow \forall x,y \in \bR^d, \dfrac{L}{2}\norm{y}^2-f(y) \geq \dfrac{L}{2}\norm{x}^2-f(x) + \la Lx-\nabla f(x), y-x \ra \\
&\Leftrightarrow \forall x,y \in \bR^d, f(y) \leq f(x)+\la \nabla f(x), y-x \ra + \dfrac{L}{2}\left(\norm{y}^2+\norm{x}^2-2\la x,y\ra\right) \\
&\Leftrightarrow \forall x,y \in \bR^d, f(y) \leq f(x)+\la \nabla f(x), y-x \ra + \dfrac{L}{2}\norm{x-y}^2 \\
&\Leftrightarrow \nabla f \textrm{ is $L$-Lipschitz continuous.}
\end{align*}
where the last equivalence comes from Theorem 2.1.5 of~\citet{nesterov}.
\end{proof}

\begin{proof}[Proof of Proposition~\ref{prop:fast_approx}]
Since $\rho$ is $\zeta$-strongly convex then $\nabla \rho^*$ is $1/\zeta$-Lipschitz continuous (see for example Theorem 4.2.1 at page 82 in~\citet{hiriart2}). Since $\rho^*$ is also convex, Lemma~\ref{lemma:conv} shows that $g:x\mapsto \frac{1}{2\zeta}\norm{x}^2-\rho^*(x)$ is convex.

Let us now consider the bin $b$ and the function $\mu=(\mu_1,\dots,\mu_k)$. Jensen's inequality gives:
\[
\dfrac{1}{|b|}\int_b g(-\mu(x)/\lambda) \dd x \geq g\left( \dfrac{1}{|b|}\int_b -\dfrac{\mu(x)}{\lambda} \dd x \right).
\]
This leads to
\begin{align*}
\int_b g(-\mu(x)/\lambda) \dd x &\geq \int_b g(-\bar{\mu}(b)/\lambda) \dd x \\
\int_b \dfrac{1}{2\zeta} \norm{-\mu(x)}^2/\lambda^2-\rho^*(-\mu(x)/\lambda) \dd x &\geq \int_b \dfrac{1}{2\zeta}\norm{-\bar{\mu}(b)}^2/\lambda^2-\rho^*(-\bar{\mu}(b)/\lambda) \dd x \\
\int_b \rho^*(-\mu(x)/\lambda)-\rho^*(-\bar{\mu}(b)/\lambda) \dd x &\leq  \dfrac{1}{2\zeta\lambda^2}\int_b \norm{\mu(x)}^2-\norm{\bar{\mu}(b)}^2 \dd x.
\end{align*}

We use the fact that $\int_b \norm{\mu(x)-\bar{\mu}(b)}^2 \dd x=\int_b \norm{\mu(x)}^2+\norm{\bar{\mu}(b)}^2-2\langle \mu(x), \bar{\mu}(b) \rangle \dd x = \int_b \norm{\mu(x)}^2+\norm{\bar{\mu}(b)}^2 \dd x - 2\langle \bar{\mu}(b), \int_b \mu(x) \dd x \rangle = \int_b \norm{\mu(x)}^2+\norm{\bar{\mu}(b)}^2 \dd x - 2\langle \bar{\mu}(b), |b| \bar{\mu}(b) \rangle = \int_b \norm{\mu(x)^2}-\norm{\bar{\mu}(b)}^2 \dd x$ and we get finally

\[
\int_b \rho^*(-\mu(x)/\lambda)-\rho^*(-\bar{\mu}(b)/\lambda) \dd x \leq  \dfrac{1}{2\zeta\lambda^2}\int_b \norm{\mu(x)-\bar{\mu}(b)}^2 \dd x.
\]
Equation~\eqref{eq:convconj} shows that
\begin{align*}
L(\pt\st)-L(p\st)&\leq \dfrac{1}{2\zeta\lambda}\sum_{b\in\cB} \int_b \norm{\bar{\mu}(b)-\mu(x)}^2  \dd x \\
&\leq \sum_{b\in\cB} \int_b \dfrac{L_{\beta}K}{2\zeta\lambda} \left(\dfrac{\sqrt{d}}{B}\right)^{2\beta} \dd x \\
& \leq \dfrac{L_{\beta}Kd^{\beta}}{2\zeta\lambda} \left(\dfrac{1}{B}\right)^{2\beta}
\end{align*}

because each $\mu_k$ is $(L_{\beta}, \beta)$-Hölder.

\end{proof}

\begin{proof}[Proof of Theorem~\ref{thm:fast}]
We denote again by $C_k$ the constants.
We sum the approximation and the estimation errors (given in Propositions~\ref{prop:fast_approx} and~\ref{prop:fast_estim}) to obtain the following bound on the regret:
\[
R(T) \leq C_1 \dfrac{\Lb Kd^{\beta}}{\zeta \lambda} B^{-2\beta}+ C_2\dfrac{\log^2(T)}{T}B^d\left(\dfrac{1}{\zeta\lambda\eta^3}+\dfrac{K}{\zeta^2\lambda^2\eta^4}+\lambda\zeta\eta^2+\lambda S\right) + 4B^d(1+\normi{\lambda\rho})\exp\left(-\dfrac{T}{12B^d}\right)
.
\]
For the sake of clarity let us note $\xi_1 \doteq C_1\dfrac{\Lb Kd^{\beta}}{\zeta \lambda}$ and $\xi_2 \doteq C_2 \left(\dfrac{1}{\zeta\lambda\eta^3}+\dfrac{K}{\zeta^2\lambda^2\eta^4}+\lambda\zeta\eta^2+\lambda S\right)$.

We have
\[
R(T) \leq \xi_1 B^{-2\beta} + \xi_2 B^d\dfrac{\log^2(T)}{T}+4B^d(1+\normi{\lambda\rho})\exp\left(-\dfrac{T}{12B^d}\right)
.\]

Taking
\[
B=\left(\dfrac{2\xi_1\beta}{\xi_2}\right)^{1/(2\beta+d)}\left(\dfrac{T}{\log^2(T)} \right)^{1/(d+2\beta)},
\]
we notice that the third term is negligible and we conclude that
\[
R(T) \leq \bigo\left(2\xi_1 \left(\dfrac{2\xi_1\beta}{\xi_2}\right)^{-2\beta/(2\beta+d)} \left(\dfrac{T}{\log^2(T)}\right)^{-2\beta/(2\beta+d)}\right)
.
\]
\end{proof}

\section{PROOFS OF INTERMEDIATE RATES}
\label{app:int}

We begin with a lemma on convex conjugates.
\begin{lemma}
\label{lemma:convconj}
Let $\lambda, \mu >0$ and let $y\in\bR^n$ and $\rho$ a non-negative bounded convex function.
Then
\[
(\lambda \rho)^*(y)-(\mu \rho)^*(y)\leq \abs{\lambda-\mu}\norm{\rho}_{\infty}.
\]
\end{lemma}

\begin{proof}
$(\lambda \rho)^*(y)=\sup_x \la x,y\ra-\lambda\rho(x)=\la x_{\lambda},y \ra-\lambda \rho(x_{\lambda})$.

And $(\mu \rho)^*(y)=\sup_x \la x,y\ra-\mu\rho(x)=\la x_{\mu},y \ra-\mu \rho(x_{\mu}) \geq \la x_{\lambda},y \ra - \mu \rho(x_{\lambda})$.

Then, $(\lambda \rho)^*(y)-(\mu \rho)^*(y) \leq \la x_{\lambda},y \ra-\lambda \rho(x_{\lambda}) - (\la x_{\lambda},y \ra - \mu \rho(x_{\lambda})) = (\mu-\lambda)\rho(x_{\lambda})$.

Finally $(\lambda \rho)^*(y)-(\mu \rho)^*(y)\leq \abs{\lambda-\mu}\norm{\rho}_{\infty}$.
\end{proof}

\begin{proof}[Proof of Proposition~\ref{prop:slow_approx_var}]
There exists $x_0 \in b$ such that $\lambb=\lambda(x_0)$ and $x_1 \in b$ such that $\mubb=\mu(x_1)$.
We use Lemma~\ref{lemma:convconj} to derive a bound for the approximation error.
\begin{align*}
&\int_b (\lambda(x)\rho)^*\left(-\mu(x)\right)-(\bar{\lambda}(b)\rho)^*\left(-\bar{\mu}(b)\right) \dd x \\
 &
= \int_b (\lambda(x)\rho)^*\left(-\mu(x)\right)-(\lambda(x)\rho)^*\left(-\bar{\mu}(b)\right) \dd x + \int_b (\lambda(x)\rho)^*\left(-\bar{\mu}(b)\right)-(\bar{\lambda}(b)\rho)^*\left(-\bar{\mu}(b)\right) \dd x  \\
& \leq \int_b \lambda(x) \left( \rho^*\left(-\dfrac{\mu(x)}{\lambda(x)}\right) - \rho^*\left(-\dfrac{\bar{\mu}(b)}{\lambda(x)}\right)\right) \dd x + \int_b \abs{\lambda(x)-\bar{\lambda}(b)} \norm{\rho}_{\infty}\dd x \\
& \leq \int_b \lambda(x) \abs*{\dfrac{\mu(x)}{\lambda(x)}-\dfrac{\bar{\mu}(b)}{\lambda(x)}} \dd x + \norm{\rho}_{\infty} \int_b \abs{\lambda(x)-\lambda(x_0)} \dd x \\
& \leq \int_b L_{\beta} |x-x_1|^{\beta} \dd x + \norm{\rho}_{\infty} \int_b \norm{\lambda'}_{\infty} \abs{x-x_0} \dd x \\
& \leq B^{-d}\left(\Lb d^{\beta/2}B^{-\beta}+\normi{\rho}\normi{\lambda'}\sqrt{d}B^{-1} \right)=\bigo(B^{-\beta-d}).
\end{align*}
\end{proof}

\begin{proof}[Proof of Proposition~\ref{prop:fast_approx_var}]
As in the proof of Proposition~\ref{prop:fast_approx} we consider a bin $b\in \cB$ and the goal is to bound
\[
\int_b \lambda(x)\rho^*\left(-\dfrac{\mu(x)}{\lambda(x)}\right)-\bar{\lambda}(b)\rho^*\left(-\dfrac{\bar{\mu}(b)}{\bar{\lambda}(b)}\right) \dd x.
\]

We use a similar method and we apply Jensen inequality with density $\dfrac{\lambda(x)}{|b|\bar{\lambda}(b)}$ to the function $g:x\mapsto \frac{1}{2\zeta}\norm{x}^2-\rho^*(x)$ which is convex.
\begin{align*}
g\left(\int_b -\dfrac{\mu(x)}{\lambda(x)} \dfrac{\lambda(x)}{|b|\bar{\lambda}(b)}\dd x\right) &\leq \int_b g\left(-\dfrac{\mu(x)}{\lambda(x)}\right) \dfrac{\lambda(x)}{|b|\bar{\lambda}(b)} \dd x \\
g\left(-\dfrac{\bar{\mu}(b)}{\bar{\lambda}(b)} \right)&\leq \int_b g\left(-\dfrac{\mu(x)}{\lambda(x)}\right) \dfrac{\lambda(x)}{|b|\bar{\lambda}(b)} \dd x \\
\dfrac{1}{2\zeta}\norm{-\dfrac{\bar{\mu}(b)}{\bar{\lambda}(b)}}^2-\rho^*\left(-\dfrac{\bar{\mu}(b)}{\bar{\lambda}(b)}\right) &\leq \dfrac{1}{|b|\bar{\lambda}(b)}\int_b \left[\dfrac{1}{2\zeta}\norm{-\dfrac{\mu(x)}{\lambda(x)}}^2-\rho^*\left(-\dfrac{\mu(x)}{\lambda(x)}\right)\right] \lambda(x)\dd x \\
\int_b \lambda(x)\rho^*\left(-\dfrac{\mu(x)}{\lambda(x)}\right)-\bar{\lambda}(b)\rho^*\left(-\dfrac{\bar{\mu}(b)}{\bar{\lambda}(b)}\right) \dd x &\leq \dfrac{1}{2\zeta} \int_b \dfrac{\norm{\mu(x)}^2}{\lambda(x)}-\dfrac{\norm{\bar{\mu}(b)}^2}{\bar{\lambda}(b)} \dd x.
\end{align*}

Consequently we have proven that
\begin{align*}
\int_b \lambda(x)\rho^*\left(-\dfrac{\mu(x)}{\lambda(x)}\right)-\bar{\lambda}(b)\rho^*\left(-\dfrac{\bar{\mu}(b)}{\bar{\lambda}(b)}\right) \dd x &\leq \dfrac{1}{2\zeta} \int_b \dfrac{\norm{\mu(x)}^2}{\lambda(x)}-\dfrac{\norm{\bar{\mu}(b)}^2}{\bar{\lambda}(b)} \dd x \\
& \leq \dfrac{1}{2\zeta} \sum_{k=1}^K \int_b \dfrac{\mu_k(x)^2}{\lambda(x)}-\dfrac{\bar{\mu}_k(b)^2}{\bar{\lambda(b)}}\dd x.
\end{align*}

Therefore we have to bound, for each $k$, $I=\displaystyle \int_b \dfrac{\mu_k(x)^2}{\lambda(x)}-\dfrac{\bar{\mu}_k(b)^2}{\bar{\lambda}(b)}\dd x$.

Let us omit the subscript $k$ and consider a $\beta$-Hölder function $\mu$.

We have
\begin{align*}
I&=\int_b \dfrac{\mu(x)^2}{\lambda(x)}-\dfrac{\bar{\mu}(b)^2}{\bar{\lambda}(b)}\dd x \\
&=\int_b \dfrac{\mu(x)^2}{\lambda(x)}-\dfrac{\mu(x)^2}{\bar{\lambda}(b)}+\dfrac{\mu(x)^2}{\bar{\lambda}(b)}-\dfrac{\bar{\mu}(b)^2}{\bar{\lambda}(b)}\dd x \\
&=
\underbrace{\int_b \left(\mu(x)^2-\bar{\mu}(b)^2\right)\left(\dfrac{1}{\lambda(x)}-\dfrac{1}{\bar{\lambda}(b)}\right) \dd x}_{I_1}
+ \underbrace{\int_b \mubb^2 \left( \dfrac{1}{\lambda(x)}-\dfrac{1}{\bar{\lambda}(b)}\right) \dd x}_{I_2}
+ \underbrace{\int_b \dfrac{1}{\lambb} \left(\mu(x)^2-\mubb^2\right) \dd x}_{I_3}
.
\end{align*}

We now have to bound these three integrals.
\medskip

\textbf{Bounding $I_1$:}
\begin{align*}
I_1&=\int_b \left( \mu(x)^2-\mubb^2\right)\left(\dfrac{1}{\lambda(x)}-\dfrac{1}{\lambb} \right) \dd x \\
&=\int_b \left( \mu(x)+\mubb\right)\left( \mu(x)-\mubb\right)\left(\dfrac{1}{\lambda(x)}-\dfrac{1}{\lambb} \right) \dd x \\
&\leq\int_b 2\abs{\mu(x)-\mubb}\abs*{\dfrac{1}{\lambda(x)}-\dfrac{1}{\lambb}} \dd x \\
&\leq 2\Lb\left(\dfrac{\sqrt{d}}{B}\right)^{\beta}\int_b\abs*{\dfrac{1}{\lambda(x)}-\dfrac{1}{\lambb}} \dd x
.
\end{align*}

Since $1/\lambda$ is of class $\cC^1$, Taylor-Lagrange inequality yields, using the fact that there exists $x_0 \in b$ such that $\lambb=\lambda(x_0)$:
\[
\abs*{\dfrac{1}{\lambda(x)}-\dfrac{1}{\lambb}}\leq \norm{\left(\dfrac{1}{\lambda} \right)'}_{\infty}\abs{x-x_0}\leq\dfrac{\norm{\lambda'}_{\infty}}{\lambda_{\min}^2}\dfrac{\sqrt{d}}{B}
.
\]

We obtain therefore
\[
I_1 \leq 2\Lb\norm{\lambda'}_{\infty}\sqrt{d}^{\beta+1}\dfrac{1}{\lambda_{\min}^2}B^{-(1+\beta+d)}=\bigo\left( \dfrac{B^{-(1+\beta+d)}}{\lambda_{\min}^2} \right)
.
\]

\textbf{Bounding $I_2$:}
\medskip

We have
\[I_2=\mubb^2\int_b \left(\dfrac{1}{\lambda(x)}-\dfrac{1}{\lambb} \right) \dd x\leq \int_b \left(\dfrac{1}{\lambda(x)}-\dfrac{1}{\lambb} \right) \dd x\]

because $\displaystyle \int_b \left(\dfrac{1}{\lambda(x)}-\dfrac{1}{\lambb} \right) \dd x \geq 0$ from Jensen's inequality.

Without loss of generality we can assume that the bin $b$ is the closed cuboid $[0, 1/B]^d$. We suppose that for all $x\in b$, $\lambda(x) >0$.

Since $\lambda$ is of class $\cC^{\infty}$, we have the following Taylor series expansion:

\[
\lambda(x)=\lambda(0)+\sum_{i=1}^d \dfrac{\partial \lambda(0)}{\partial x_i}x_i + \dfrac{1}{2}\sum_{i,j} \dfrac{\partial^2 \lambda(0)}{\partial x_i \partial x_j} x_i x_j + \smallo(\norm{x}^2)
.
\]

Integrating over the bin $b$ we obtain

\[
\lambb = \lambda(0)+\dfrac{1}{2}\dfrac{1}{B}\sum_{i=1}^d \dfrac{\partial \lambda(0)}{\partial x_i} + \dfrac{1}{8}\dfrac{1}{B^2}\sum_{i\neq j} \dfrac{\partial^2 \lambda(0)}{\partial x_i \partial x_j} + \dfrac{1}{6}\dfrac{1}{B^2} \sum_{i=1}^d \dfrac{\partial^2 \lambda(0)}{\partial x_i^2} + \smallo\left(\dfrac{1}{B^2}\right)
.
\]

Consequently
\begin{align*}
\int_b \dfrac{\dd x}{\lambb} &= \dfrac{1}{B^d \lambb} \\
&=\dfrac{1}{B^d \lambda(0)} \dfrac{1}{1+\dfrac{1}{2\lambda(0)}\dfrac{1}{B} \displaystyle \sum_{i=1}^d \dfrac{\partial \lambda(0)}{\partial x_i} + \dfrac{1}{\lambda(0)}\dfrac{1}{B^2}\left(\dfrac{1}{8} \displaystyle \sum_{i\neq j} \dfrac{\partial^2 \lambda(0)}{\partial x_i \partial x_j} + \dfrac{1}{6} \sum_{i=1}^d \dfrac{\partial^2 \lambda(0)}{\partial x_i^2}\right) + \smallo\left(\dfrac{1}{B^2}\right)}\\
&=\dfrac{1}{B^d \lambda(0)} \Biggl(1- \dfrac{1}{2\lambda(0)}\dfrac{1}{B} \displaystyle \sum_{i=1}^d \dfrac{\partial \lambda(0)}{\partial x_i} - \dfrac{1}{\lambda(0)}\dfrac{1}{B^2}\left(\dfrac{1}{8} \displaystyle \sum_{i\neq j} \dfrac{\partial^2 \lambda(0)}{\partial x_i \partial x_j} + \dfrac{1}{6} \sum_{i=1}^d \dfrac{\partial^2 \lambda(0)}{\partial x_i^2}\right) \\
& \phantom{aaaaaaaaaaaaaaaaaaaaaaaaaaaaaa} + \dfrac{1}{4\lambda(0)^2} \dfrac{1}{B^2} \left(\sum_{i=1}^d \dfrac{\partial \lambda(0)}{\partial x_i} \right)^2 + \smallo\left(\dfrac{1}{B^2}\right) \Biggr) \\
&=\dfrac{1}{B^d \lambda(0)}-\dfrac{1}{2\lambda(0)^2}\dfrac{1}{B^{d+1}} \displaystyle \sum_{i=1}^d \dfrac{\partial \lambda(0)}{\partial x_i} - \dfrac{1}{\lambda(0)^2}\dfrac{1}{B^{d+2}}\left(\dfrac{1}{8} \displaystyle \sum_{i\neq j} \dfrac{\partial^2 \lambda(0)}{\partial x_i \partial x_j} + \dfrac{1}{6} \sum_{i=1}^d \dfrac{\partial^2 \lambda(0)}{\partial x_i^2}
\right)\\
& \phantom{aaaaaaaaaaaaaaaaaaaaaaaaaaaaaa} + \dfrac{1}{4\lambda(0)^3} \dfrac{1}{B^{d+2}} \left(\sum_{i=1}^d \dfrac{\partial \lambda(0)}{\partial x_i} \right)^2 + \smallo\left(\dfrac{1}{B^2}\right)
.
\end{align*}

Let us now compute the Taylor series development of $1/\lambda$. We have:

\[
\dfrac{\partial}{\partial x_i}\dfrac{1}{\lambda(x)}=-\dfrac{1}{\lambda(x)^2}\dfrac{\partial \lambda(x)}{\partial x_i} \quad \textrm{and} \quad \dfrac{\partial^2}{\partial x_i \partial x_j}\dfrac{1}{\lambda(x)}=-\dfrac{1}{\lambda(x)^2}\dfrac{\partial^2\lambda(x)}{\partial x_i \partial x_j}+\dfrac{2}{\lambda(x)^3}\dfrac{\partial \lambda(x)}{\partial x_i}\dfrac{\partial \lambda(x)}{\partial x_j}
.
\]

This lets us write
\begin{align*}
\dfrac{1}{\lambda(x)}&=\dfrac{1}{\lambda(0)}-\dfrac{1}{\lambda(0)^2}\sum_{i=1}^d \dfrac{\partial \lambda(0)}{\partial x_i}x_i - \dfrac{1}{2}\dfrac{1}{\lambda(0)^2} \sum_{i,j}\dfrac{\partial^2 \lambda(0)}{\partial x_i \partial x_j}x_i x_j + \dfrac{1}{\lambda(0)^3}\sum_{i,j}\dfrac{\partial \lambda(0)}{\partial x_i}\dfrac{\partial \lambda(0)}{\partial x_j}x_i x_j + \smallo(\norm{x}^2) \\
\int_b \dfrac{\dd x}{\lambda(x)}&=\dfrac{1}{\lambda(0)}\dfrac{1}{B^d}-\dfrac{1}{2\lambda(0)^2}\dfrac{1}{B^{d+1}}\sum_{i=1}^d \dfrac{\partial \lambda(0)}{\partial x_i}-\dfrac{1}{\lambda(0)^2}\dfrac{1}{B^{d+2}}\left(\dfrac{1}{8}\sum_{i\neq j}\dfrac{\partial^2 \lambda(0)}{\partial x_i \partial x_j}+\dfrac{1}{6}\sum_{i=1}^d \dfrac{\partial^2\lambda(0)}{\partial x_i^2} \right) \\
&\phantom{aaaaaaaaaa}+\dfrac{1}{\lambda(0)^3}\dfrac{1}{B^{d+2}}\left( \dfrac{1}{4}\sum_{i\neq j}\dfrac{\partial \lambda(0)}{\partial x_i}\dfrac{\partial \lambda(0)}{\partial x_j} + \dfrac{1}{3}\sum_{i=1}^d\left(\dfrac{\partial \lambda(0)}{\partial x_i}\right)^2\right)+\smallo\left(\dfrac{1}{B^{d+2}}\right).
\end{align*}

And then

\[
I_2 \leq \dfrac{1}{12}\dfrac{1}{\lambda(0)^3}\dfrac{1}{B^{d+2}}\sum_{i=1}^d \left(\dfrac{\partial \lambda(0)}{\partial x_i}\right)^2+\smallo\left(\dfrac{1}{B^{d+2}}\right)
.
\]

Since the derivatives of $\lambda$ are bounded we obtain that

\[
I_2=\bigo\left(\dfrac{B^{-2-d}}{\lambda_{\min}^3}\right)
.
\]

\textbf{Bounding $I_3$:}
\begin{align*}
I_3 &=\int_b \dfrac{1}{\lambb}\left(\mu(x)^2-\mubb^2 \right) \dd x \\
&=\dfrac{1}{\lambb}\int_b \left( \mu(x)-\mubb \right)^2 \dd x \\
&\leq \dfrac{1}{\lambda_{\min}}\Lb^2d^{\beta}B^{-(2\beta+d)}=\bigo\left(\dfrac{B^{-(2\beta+d)}}{\lambda_{\min}} \right)
.
\end{align*}

Putting this together we have $I=\bigo\left((d\Lb^2\normi{\nabla \lambda}^2) \dfrac{B^{-(2\beta+d)}}{\lambda_{\min}^3} \right)$. And finally
\[
L(\pt\st)-L(p\st) \leq \bigo\left( Kd\Lb^2\normi{\nabla \lambda}^2\dfrac{B^{-2\beta}}{\zeta\lambda_{\min}^3} \right)
.
\]
\end{proof}

\begin{lemma}[Regularity of $\eta$]
\label{lemma:reg_eta}
If $\eta$ is the distance of the optimum $p\st$ to the boundary of $\Dk$ as defined in Definition~\ref{def:eta}, and if the $\mu_k$ functions are all $\beta$-Hölder and $\lambda$ of class $\cC^1$, then $\eta$ is $\beta$-Hölder. More precisely we have
\[
\forall x,y \in b, \abs*{\eta(x)-\eta(y)} \leq \sqrt{\dfrac{K}{K-1}}\dfrac{\normi{\lambda}+\normi{\lambda'}}{\zeta \lambda_{\min}(b)^2}\abs{x-y}^{\beta}=\dfrac{C_L}{\lambda_{\min}(b)^2}\abs{x-y}^{\beta}.
\]
\end{lemma}

\begin{proof}
Let $x \in \cX$. Since $\eta(x)=\dist(p_b\st,\partial \Dk)$ we obtain
\[
\eta(x)=\sqrt{\dfrac{K}{K-1}}\min_i p_i\st(x)
.
\]
And
\begin{align*}
p\st(x)&=\argmin \la \mu(x), p(x) \ra + \lambda(x) \rho(p(x)) \\
&= \nabla (\lambda(x) \rho)^*(-\mu(x)) \\
&= \nabla \rho^*\left( -\dfrac{\mu(x)}{\lambda(x)}\right)
.
\end{align*}

Since $\rho$ is $\zeta$-strongly convex, $\nabla \rho^*$ is $1/\zeta$-Lipschitz continuous.

Therefore, for $x, y \in b$,
\begin{align*}
\abs{p\st(x)-p\st(y)} &\leq \dfrac{1}{\zeta}\abs*{\dfrac{\mu(x)}{\lambda(x)}-\dfrac{\mu(y)}{\lambda(y)}} \\
&\leq \dfrac{1}{\zeta} \abs*{\dfrac{\mu(x)-\mu(y)}{\lambda(x)}} + \dfrac{1}{\zeta} \abs{\mu(y)} \abs*{\dfrac{1}{\lambda(x)}-\dfrac{1}{\lambda(y)}} \\
&\leq \dfrac{1}{\zeta \lambda_{\min}(b)} \abs{x-y}^{\beta} + \dfrac{1}{\zeta}\dfrac{\norm{\lambda'}_{\infty}}{\lambda_{\min}(b)^2}\abs{x-y}
\end{align*}
since all $\mu_k$ are bounded by $1$ (the losses are bounded by $1$).
\end{proof}

\begin{proof}[Proof of Lemma~\ref{lemma:wb}]
We consider a well-behaved bin $b$. There exists $x_1 \in b$ such that $\lambda(x_1) \geq c_1B^{-\beta/3}$. Since $\lambda$ is $\cC^{\infty}$ on $[0,1]^d$, it is in particular Lipschitz-continuous on $b$. And therefore
\[
\forall x \in b, \ \lambda(x) \geq c_1B^{-\beta/3} - \norm{\lambda'}_{\infty} \diam(b) \geq c_1B^{-\beta/3} - \norm{\lambda'}_{\infty} \diam(b)^{\beta/3}=B^{-\beta/3}
.
\]

Lemma~\ref{lemma:reg_eta} shows that $\eta$ is $\beta$-Hölder continuous (with constant denoted by $C_L/\lambda_{\min}^2$) and therefore we have

\[
\forall x \in b, \ \eta(x) \geq c_2B^{-\beta/3} - \dfrac{C_L}{\lambda_{\min}(b)^2}\diam(b)^{\beta}=B^{-\beta/3}.
\]
\end{proof}

\begin{lemma}
\label{lemma:eta_lambda}
If $\rho$ is convex, $\eta$ is an increasing function of $\lambda$.
\end{lemma}

\begin{proof}
As in the proof of Proposition~\ref{prop:slow_entropy} we use the KKT conditions to find that on a bin $b$ (without the index $k$ for the arm):
\[
\mubb+\lambb\nabla\rho(p_b\st)+\xi=0.
\]
Therefore \[p\st_b=(\nabla\rho)^{-1}\left(-\dfrac{\xi+\mubb}{\lambb}\right).\]
Since $\rho$ is convex, $\nabla \rho$ is an increasing function and its inverse as well. Consequently $p_b\st$ is an increasing function of $\lambb$, and since $\eta(b)=\sqrt{K/(K-1)}\min_i{p\st_{b,i}}$, $\eta$ is also an increasing function of $\lambb$.
\end{proof}

\begin{proof}[Proof of Theorem~\ref{thm:int}]
Since $B$ will be chosen as an increasing function of $T$ we only consider $T$ sufficiently large in order to have $c_1 B^{-\beta/3}< \delta_1$ and $c_2 B^{-\beta/3}< \delta_2$. To ensure this we can also take smaller $\delta_1$ and $\delta_2$. Moreover we lower the value of $\delta_2$ or $\delta_1$ to be sure that $\frac{\delta_2}{c_2}=\eta(\frac{\delta_1}{c_1})$. These are technicalities needed to simplify the proof.

The proof will be divided into several steps. We will first obtain lower bounds on $\lambda$ and $\eta$ for the ``well-behaved bins". Then we will derive bounds for the approximation error and the estimation error. And finally we will put that together to obtain the intermediate convergence rates.

As in the proofs on previous theorems we will denote the constants $C_k$ with increasing values of $k$.

\begin{itemize}
\item \textbf{Lower bounds on $\eta$ and $\lambda$:}

Using a technique from~\citet{zeevi} we notice that without loss of generality we can index the $B^d$ bins with increasing values of $\bar{\lambda}(b)$. Let us note $\mathcal{IB}=\{1, \dots, j_1\}$ and $\mathcal{WB}=\{j_1+1, \dots, B^d\}$. Since $\eta$ is an increasing function of $\lambda$ (cf Lemma~\ref{lemma:eta_lambda}), the $\eta(b_j)$ are also increasingly ordered.

Let $j_2 \geq j_1$ be the largest integer such that $\bar{\lambda}(b_j) \leq \dfrac{\delta_1}{c_1}$. Consequently we also have that $j_2$ is the largest integer such that $\eta(b_j) \leq \dfrac{\delta_2}{c_2}$.

Let $j \in \{j_1+1,\dots,j_2\}$. The bin $b_j$ is a well-behaved bin and Lemma~\ref{lemma:wb} shows that $\bar{\lambda}(b_j)\geq B^{-\beta/3}$. Then $\bar{\lambda}(b_j)+(c_1-1)B^{-\beta/3} \leq c_1 \bar{\lambda}(b_j) \leq \delta_1$ and we can apply the margin condition (cf Assumption~\ref{ass:margin}) which gives

\[
\bP_X(\lambda(x) \leq \bar{\lambda}(b_j)+(c_1-1)B^{-\beta/3}) \leq C_m (c_1 \bar{\lambda}(b_j))^{6\alpha}
.
\]

But since the context are uniformly distributed and since the $\bar{\lambda}(b_j)$ are increasingly ordered we also have that
\[
\bP_X(\lambda(x) \leq \bar{\lambda}(b_j)+(c_1-1)B^{-\beta/3}) \geq \bP_X(\lambda(x) \leq \bar{\lambda}(b_j)) \geq \dfrac{j}{B^d}
.
\]

This gives $\bar{\lambda}(b_j) \geq \dfrac{1}{c_1C_m^{1/6\alpha}}\left(\dfrac{j}{B^d}\right)^{1/6\alpha}$.
The same computations give $\eta(b_j) \geq \dfrac{1}{c_2C_m^{1/6\alpha}} \left(\dfrac{j}{B^d}\right)^{1/6\alpha}$.
We note $\Cg\doteq \min((c_1C_m^{1/6\alpha})^{-1},(c_2C_m^{1/6\alpha})^{-1}))$ and $\gamma_j \doteq \Cg\left(\dfrac{j}{B^d}\right)^{1/\alpha}$. Consequently $\bar{\lambda}(b_j) \geq \gamma_j$ and $\eta(b_j) \geq \gamma_j$.

Let us now compute the number of ill-behaved bins:
\begin{align*}
\#\{ b \in \cB, b \notin \mathcal{WB}\} &= B^d \ \bP(b \notin \mathcal{WB}) \\
&= B^d \ \bP(\forall x \in \cB, \ \eta(x) \leq c_2 B^{-\beta/3} \textrm{ or } \forall x \in \cB, \ \lambda(x) \leq  c_1 B^{-\beta/3}) \\
&\leq B^d \ \bP(\eta(\bar{x}) \leq c_2 B^{-\beta/3} \textrm{ or } \lambda(\bar{x}) \leq c_1 B^{-\beta/3}) \\
&\leq C_{m}(c_1^{6\alpha}+c_2^{6\alpha})B^d B^{-2\alpha \beta} \doteq C_I B^d B^{-2\alpha \beta}
\end{align*}
where $\bar{x}$ is the mean context value in the bin $b$.
Consequently if $j \geq j\st\doteq C_{I}B^d B^{-2\alpha \beta}$, then $b_j \in \mathcal{WB}$. Let $\hat{j}\doteq C_{I}B^d B^{-\alpha \beta} \geq j\st$. Consequently for all $j \geq j\st$, $b_j \in \mathcal{WB}$.

We want to obtain an upper-bound on the constant $S\bar{\lambda(b_j)}+\dfrac{K}{\eta(b_j)^4\bar{\lambda}(b_j)^2}$ that arises in the fast rate for the estimation error. For the sake of clarity we will remove the dependency in $b_j$ and denote this constant $C=S\lambda+\dfrac{K}{\lambda^2 \eta^4}$.

In the case of the entropy regularization $S=1/\min_i p\st_i$. Since $\eta=\sqrt{K/(K-1)}\min_i p\st_i$, we have that $\min_i p_i\st=\sqrt{(K-1)/K}\eta\geq \eta/2$. Consequently $S \leq 2/\gamma_j$ and, on a well-behaved bin $b_j$, for $j \leq j_2$,
\[C \leq \dfrac{K+2\normi{\lambda}}{\gamma_j^6}\doteq \dfrac{C_F}{\gamma_j^6}, \numberthis \label{eq:C}\]

where the subscript $F$ stands for ``Fast".
When $j\geq j_2$, we have $\bar{\lambda}(b_j)\geq \delta_1/c_1$ and $\eta(b_j)\geq \delta_2/c_2$ and consequently \[C \leq \dfrac{K}{(\delta_1/c_1)^2(\delta_2/c_2)^4}+\dfrac{2\normi{\lambda}}{\delta_2/c_2}\doteq C_{\max}.\]

Let us notice than $\lambda$ being known by the agent, the agent knows the value of $\lambb$ on each bin $b$ and can therefore order the bins. Consequently the agent can sample, on every well-behaved bin, each arm $T\gamma_j/2$ times and be sure that $\min_i p_i \geq \gamma_j/2$.
On the first $\floor{\hat{j}}$ bins the agent will sample each arm $\lambb\sqrt{T/B^d}$ times as in the proof of Proposition~\ref{prop:slow_entropy}.

\item \textbf{Approximation Error:}

We now bound the approximation error. We separate the bins into two sets: $\{1, \dots, \floor{j\st} \}$ and $\{\ceil{j\st},\dots,B^d\}$. On the first set we use the slow rates of Proposition~\ref{prop:slow_approx_var} and on the second set we use the fast rates of Proposition~\ref{prop:fast_approx_var}.

We obtain that, for $\alpha <1/2$,
\begin{align*}
L(\pt\st)-L(p\st) &\leq \Lb d^{\beta/2}\sum_{j=1}^{\floor{j\st}}B^{-\beta-d}+\normi{\rho}\normi{\nabla\lambda}\sqrt{d}\sum_{j=1}^{\floor{j\st}}B^{-1-d}+ (Kd\Lb^2\normi{\nabla\lambda}^2)\sum_{j=\ceil{j\st}}^{B^d}\dfrac{B^{-2\beta-d}}{\bar{\lambda}(b_j)^3} \\
&\leq C_{I}\Lb d^{\beta/2} B^{-\beta}B^{-2\alpha\beta}+ (Kd\Lb^2\normi{\nabla\lambda}^2)\left(\sum_{j=\ceil{j\st}}^{j_2}\dfrac{B^{-2\beta-d}}{\gamma_j^3} +\sum_{j=j_2+1}^{B^d}\dfrac{B^{-2\beta-d}}{(c_1/\delta_1)^3}\right) +\smallo (B^{-2\alpha\beta-\beta}) \\
&\leq C_{I}\Lb d^{\beta/2} B^{-2\alpha\beta-\beta}+(Kd\Lb^2\normi{\nabla\lambda}^2)\left(\dfrac{B^{-2\beta-d}}{\Cg^3}\sum_{j=\ceil{j\st}}^{j_2} \left(\dfrac{j}{B^d}\right)^{-1/2\alpha}+ B^{-2\beta}\left(\dfrac{\delta_1}{c_1}\right)^3\right) +\smallo (B^{-2\alpha\beta-\beta})\\
&\leq C_{I}\Lb d^{\beta/2} B^{-2\alpha\beta-\beta}+ (Kd\Lb^2\normi{\nabla\lambda}^2)\dfrac{1}{\Cg^3} B^{-2\beta}\int_{C_{I}B^{-2\alpha \beta}}^1 x^{-1/2\alpha}\dd x +\smallo (B^{-2\alpha\beta-\beta})\\
&\leq\left(C_{I}\Lb d^{\beta/2} + Kd\Lb^2\normi{\nabla\lambda}^2\dfrac{2\alpha}{1-2\alpha} \dfrac{C_I^{(2\alpha-1)/2\alpha}}{\Cg^3}\right)B^{-\beta-2\alpha\beta} +\smallo (B^{-2\alpha\beta-\beta})=\bigo\left(B^{-\beta-2\alpha\beta}\right)
\end{align*}
since $\alpha<1/2$. We step from line 3 to 4 thanks to a series-integral comparison.

For $\alpha=1/2$ we get
\[
L(\pt\st)-L(p\st) \leq \left(C_{I}\Lb d^{\beta/2}+\left(Kd\Lb^2\normi{\nabla\lambda}^2\right)(\delta_1^3c_1^{-3}+2\beta\Cg^{-3}\log(B))\right) B^{-2\beta}+\smallo(B^{-2\beta}) =\bigo\left(B^{-2\beta}\log(B)\right).
\]

And for $\alpha>1/2$ we have
\[
L(\pt\st)-L(p\st) \leq\left(Kd\Lb^2\normi{\nabla\lambda}^2\right)\left(\dfrac{1}{\Cg^3}\dfrac{2\alpha}{2\alpha-1}+ \left(\dfrac{\delta_1}{c_1}\right)^3\right) B^{-2\beta}+\smallo (B^{-2\beta}) = \bigo \left( B^{-2\beta} \right)
\]
because $\beta+2\alpha\beta>2\beta$.

Let us note 
\begin{align*}
\xi_1&\doteq \left(C_{I}\Lb d^{\beta/2} + Kd\Lb^2\normi{\nabla\lambda}^2\dfrac{2\alpha}{1-2\alpha} \dfrac{C_I^{(2\alpha-1)/2\alpha}}{\Cg^3}\right);\\
\xi_2&\doteq \left(C_{I}\Lb d^{\beta/2}+\left(Kd\Lb^2\normi{\nabla\lambda}^2\right)(\delta_1^3c_1^{-3}+2\beta\Cg^{-3}\log(B))\right);\\
\xi_3&\doteq \left(Kd\Lb^2\normi{\nabla\lambda}^2\right)\left(\dfrac{1}{\Cg^3}\dfrac{2\alpha}{2\alpha-1}+ \left(\dfrac{\delta_1}{c_1}\right)^3\right); \\
\xi_{app} &\doteq \max(\xi_1,\xi_2,\xi_3).
\end{align*}

Finally we obtain that the approximation error is bounded by $\xi_{app} B^{-\min(\beta+2\alpha \beta, 2\beta)} \log(B)$ with $\alpha>0$.

\item \textbf{Estimation Error:}

We proceed in a similar manner as for the approximation error, except that we do not split the bins around $j\st$ but around $\hat{j}$.

In a similar manner to the proofs of Theorems~\ref{thm:slow} and~\ref{thm:fast} we only need to consider the terms of dominating order from Propositions~\ref{prop:slow_estim} and~\ref{prop:fast_estim}.
As before we consider the same event $A$  (cf the proof of Proposition~\ref{prop:slow_estim}) and we note $C_A \doteq 4B^d(1+\normi{\lambda\rho})$. We obtain, for $\alpha < 1$, using~\eqref{eq:C}:
\begin{align*}
\bE L(\pt_T)-L(\pt\st) &= \dfrac{1}{B^d} \sum_{b\in \cB} \bE L_b(\pt_T)-L(p_b\st) \\
&=\dfrac{1}{B^d} \sum_{j=\ceil{\hat{j}}}^{B^d} \bE L_b(\pt_T)-L(p_b\st)+\dfrac{1}{B^d} \sum_{j=1}^{\floor{\hat{j}}} \bE L_b(\pt_T)-L(p_b\st) \\
&\leq \dfrac{1}{B^d} \sum_{j=\ceil{\hat{j}}}^{B^d} 2C\dfrac{\log^2(T)}{T/B^d}+\dfrac{1}{B^d} \sum_{j=1}^{\floor{\hat{j}}} 4\sqrt{12K} \sqrt{\dfrac{\log(T)}{T/B^d}} + C_A e^{-\frac{T}{12B^d}}\\
&\leq 2C_F\sum_{j=\ceil{\hat{j}}}^{j_2} \dfrac{\log^2(T)}{T}  \gamma_j^{-6} + \sum_{j=j_2+1}^{B^d} 2 C_{\max} \dfrac{\log^2(T)}{T}+ 6\sqrt{3K}\sqrt{\dfrac{\log(T)}{T}}B^{d/2}B^{-\alpha \beta} + C_A e^{-\frac{T}{12B^d}} \\
&\leq \dfrac{2C_F}{\Cg^6} \dfrac{\log^2(T)}{T} \sum_{j=\ceil{\hat{j}}}^{j_2} \left( \dfrac{j}{B^d}\right)^{-1/\alpha} + 2C_{\max}\dfrac{\log^2(T)}{T} B^d + 6\sqrt{3K}\sqrt{\dfrac{\log(T)}{T}}B^{d/2-\alpha \beta} + C_A e^{-\frac{T}{12B^d}}\\
&\leq \dfrac{2C_F}{\Cg^6} \dfrac{\log^2(T)}{T}B^d\int_{C_I B^{-\alpha \beta}}^1 x^{-1/\alpha}\dd x  + 2C_{\max}\dfrac{\log^2(T)}{T} B^d + 6\sqrt{3K}\sqrt{\dfrac{\log(T)}{T}}B^{d/2-\alpha \beta} + C_A e^{-\frac{T}{12B^d}}\\
&\leq \dfrac{2C_F}{\Cg^6}\dfrac{\log^2(T)}{T}B^d \dfrac{\alpha}{1-\alpha}B^{\beta (1-\alpha)} + 2C_{\max}\dfrac{\log^2(T)}{T} B^d + 6\sqrt{3K}\sqrt{\dfrac{\log(T)}{T}}B^{d/2-\alpha \beta} + C_A e^{-\frac{T}{12B^d}}\\
&\leq \dfrac{2C_F}{\Cg^6}\dfrac{\log^2(T)}{T} \dfrac{\alpha}{1-\alpha}B^{d+\beta-\alpha\beta} + 6\sqrt{3K}\sqrt{\dfrac{\log(T)}{T}}B^{d/2-\alpha \beta}+2C_{\max}\dfrac{\log^2(T)}{T}B^d + C_A e^{-\frac{T}{12B^d}}.
\end{align*}

\item \textbf{Putting things together:}

We note $\Ca\doteq \dfrac{2C_F}{\Cg^6}\dfrac{\alpha}{1-\alpha}$.
This leads to the following bound on the regret:
\[
R(T) \leq  \Ca\dfrac{\log^2(T)}{T}B^{d+\beta-\alpha \beta}+6\sqrt{3K}\sqrt{\dfrac{\log(T)}{T}}B^{d/2-\alpha\beta}+2C_{\max}\dfrac{\log^2(T)}{T}B^d + C_A e^{-\frac{T}{12B^d}}+\xi_{app}B^{-\min(2\beta,\beta+2\alpha\beta)}\log(B)
.
\]
Choosing $B=\left(\dfrac{T}{\log^2(T)}\right)^{1/(2\beta+d)}$ we get
\begin{align*}
R(T) &\leq (\Ca+6\sqrt{3}K) \left(\dfrac{T}{\log^2(T)}\right)^{-\beta(1+\alpha)/(2\beta+d)} + \smallo \left(\dfrac{T}{\log^2(T)}\right)^{-\beta(1+\alpha)/(2\beta+d)}
\end{align*}
which is valid for $\alpha \in (0,1)$.

Finally we have \[R(T) = \bigo \left(\left(\dfrac{T}{\log^2(T)}\right)^{-\beta(1+\alpha)/(2\beta+d)}\right).\]

\end{itemize}

\end{proof}

\section{PROOFS OF LOWER BOUNDS}
\label{app:lb}

\begin{proof}[Proof of Theorem~\ref{thm:lowfast}]
We consider the model with $K=2$ where $\mu(x) = (-\eta(x),\eta(x))^\top$, where $\eta$ is a $\beta$-Hölder function on $\cX = [0,1]^d$. We note that $\eta$ is uniformly bounded over $\cX$ as a consequence of smoothness, so one can take $\lambda$ such that $|\eta(x)| < \lambda$. We denote by $e=(1/2,1/2)$ the center of the simplex, and we consider the loss
\[
L(p)=\int_{\cX}\big( \la \mu(x), p(x) \ra + \lambda \|p(x)-e\|^2 \big)\dd x.
\]
Denoting by $p_0(x)$ the vector $e+\mu(x)/(2\lambda)$, we have that $p_0(x) \in \Delta^2$ for all $x\in \cX$. Further, we have that
\[
 \la \mu(x), p(x) \ra + \lambda \|p(x)-e\|^2 = \lambda  \|p(x) - p_0(x)\|^2 + 1/(4\lambda)\| \mu(x)\|^2\, ,
\]
since $\la \mu(x), e \ra =0$. As a consequence, $L$ is minimized at $p_0$ and
\[
L(p) - L(p_0) = \int_{\cX} \lambda \|p(x)-p_0(x)\|^2  \dd x = 1/(2\lambda) \int_\cX |\eta(x) - \eta_0(x)|^2 \dd x\, .
\]
where $\eta$ is such that $p(x) = \big(1/2-\eta(x)/(2\lambda),1/2+\eta(x)/(2\lambda)\big)$. As a consequence, for any algorithm with final variable $\hat p_T$, we can construct an estimator $\hat \eta_T$ such that
\[
\bE[L(\hat p_T)] - L(p_0) = 1/(2\lambda) \bE \int_\cX |\hat \eta_T(x) - \eta_0(x)|^2 \dd x\, ,
\]
where the expectation is taken over the randomness of the observations $Y_t$, with expectation $\pm \eta(X_t)$, with sign depending on the known choice $\pi_t=1$ or $2$. As a consequence, any upper bound on the regret for a policy implies an upper bound on regression over $\beta$-Hölder functions in dimension $d$, with $T$ observations. This yields that, in the special case where $\rho$ is the 1-strongly convex function equal to the squared $\ell_2$ norm
\[
\inf_{\hat p} \sup_{\substack{\mu \in \mathcal{H}_{\beta} \\ \rho \, = \, \ell^2_2}} \bE[L(\hat p_T)] - L(p_0) \ge \inf_{\hat \eta} \sup_{\eta \in \mathcal{H}_{\beta}} 1/(2\lambda) \bE \int_\cX |\hat \eta_T(x) - \eta_0(x)|^2 \dd x\, \ge  C T^{-\frac{2\beta}{2\beta+d}}\, .
\]
The final bound is a direct application of Theorem 3.2 in~\citet{gyorfi}.
\end{proof}

\end{document}